\documentclass[11pt]{article}
\pdfoutput=1
\usepackage{enumerate}
\usepackage{pdfsync}
\usepackage[OT1]{fontenc}
\usepackage{smile}
\usepackage[colorlinks,
            linkcolor=red,
            anchorcolor=blue,
            citecolor=blue
            ]{hyperref}
\usepackage{color, colortbl}
\usepackage{fullpage}
\usepackage{hyperref}
\usepackage[protrusion=true,expansion=true]{microtype}
\usepackage{pbox}
\usepackage{setspace}
\usepackage{tabularx}
\usepackage{float}
\usepackage{wrapfig,lipsum}
\usepackage{enumitem}
\usepackage{microtype}
\usepackage{graphicx}
\usepackage{subfigure}
\usepackage{enumerate}
\usepackage{enumitem}
\usepackage{pgfplots}
\usepackage{comment}

\usepackage[capitalize,noabbrev]{cleveref}

\usetikzlibrary{arrows,shapes,snakes,automata,backgrounds,petri}
\usepackage{booktabs} 

\allowdisplaybreaks

\usepackage{enumitem}

\newtheorem*{unnumberedtheorem}{Theorem}
\newtheorem*{unnumberedprop}{Proposition}

\newcommand{\I}{\mathbb{I}}
\newcommand{\de}{\mathbb{E}}
\newcommand{\md}{\mathcal{D}}
\newcommand{\dP}{\mathbb{P}}

\newcommand{\ts}{\pi_{\mathrm{TS}}}
\newcommand{\ids}{\pi_{\mathrm{IDS}}}
\newcommand{\app}{\pi_{\mathrm{app}}}

\newcommand{\kl}{D_{\mathrm{KL}}}
\newcommand{\E}{\mathcal{E}}

\usepackage{colortbl}
\definecolor{LightCyan}{rgb}{0.8, 0.9, 1}
\usepackage[utf8]{inputenc} 
\usepackage[T1]{fontenc}    
\usepackage{hyperref}       
\usepackage{url}            
\usepackage{booktabs}       
\usepackage{amsfonts}       
\usepackage{nicefrac}       
\usepackage{microtype}      
\usepackage{xcolor}         
\usepackage{bbm}

\usepackage{todonotes}
\usepackage{colortbl}
\definecolor{LightCyan}{rgb}{0.8, 0.9, 1}

\usepackage{enumitem}
\usepackage[left=0.9in, right=0.9in, top=1in, bottom=1in]{geometry} 

\renewcommand{\textcolor}[2]{#2}

\makeatletter
\newcommand*{\rom}[1]{\expandafter\@slowromancap\romannumeral #1@}
\makeatother
\title{\huge Sample-Efficient Reinforcement Learning from Human Feedback via Information-Directed Sampling}

\author
{
  Han Qi\thanks{Equal contribution.}
  \thanks{Shanghai AI Laboratory. Email: {\tt zhangqiaosheng@pjlab.org.cn}}
  \thanks{Xi'an Jiaotong University. Email: {\tt qihan19@stu.xjtu.edu.cn}}
  \and
	Haochen Yang$^{*}$\thanks{Peking University. Email: {\tt hcyang@stu.pku.edu.cn}}
	\and
	Qiaosheng Zhang$^\dagger$ 
  \and
  Zhuoran Yang\thanks{Yale University. Email: {\tt zhuoran.yang@yale.edu}}
}

\begin{document}
\date{}
\maketitle

\begin{abstract}

We study the problem of reinforcement learning from human feedback (RLHF), a critical problem in training large language models, from a theoretical perspective. Our main contribution is the design of novel sample-efficient RLHF algorithms based on information-directed sampling (IDS), an online decision-making principle inspired by information theory.
Our algorithms  maximize the sum of the value function and a mutual information term that encourages exploration of the unknown environment (which quantifies the information gained about the environment through observed human feedback data). To tackle the challenge of large state spaces  and  improve sample efficiency, we construct a simplified \emph{surrogate environment} and introduce a novel distance measure (named the \emph{$\ell_g$-distance}), enabling our IDS-based algorithm to achieve a Bayesian regret upper bound of order 
$O(H^{3/2}\sqrt{\log(K(\epsilon)) T})$, where $H$ is the episode length, $T$ is the number of episode and   $K(\epsilon)$  is related to the  covering number  of the  environment. Specializing to the tabular settings, this regret bound is of order $\tilde{O}(H^2\sqrt{SAT})$, where $S$ and $A$ are the numbers of states and actions. Finally, we propose an Approximate-IDS algorithm that is computationally more efficient while maintaining nearly the same sample efficiency. The  design principle of this approximate algorithm is not only effective in RLHF settings but also applicable to the standard RL framework. Moreover, our work showcases the value of information theory in reinforcement learning and in the training of large language models.

\end{abstract}

\section{Introduction}
Reinforcement learning from human feedback (RLHF) is a key technique for aligning large language models (LLMs) to human values \cite{ouyang2022training}, and has also shown immense potential in many other fields, such as stock prediction, robot training, medical treatments \cite{zhu2023principled}. It can be viewed as an extension of standard reinforcement learning (RL) in the sense that feedback is not given as a numerical reward, but as a one-bit preference over a trajectory pair. Compared to standard RL, this preference-based setting is often more aligned with real-world scenarios, especially for tasks involving human evaluations~\cite{chen2022human}. 
However, a key challenge for applying RLHF algorithms is their reliance on extensive human feedback data, which is usually expensive and time-intensive to collect. To address this challenge, recent works on RLHF mainly focus on developing online learning methods that encourage exploration to improve sample efficiency, thereby reducing the amount of human feedback needed~\cite{xie2024exploratorypreferenceoptimizationharnessing}. This brings the RLHF problem back to a fundamental question in RL: \emph{how to effectively balance the trade-off between exploration and exploitation to improve sample efficiency?}

To tackle this trade-off, two major design principles have been introduced. The first approach, \emph{Optimism in the Face of Uncertainty} (OFU), typically relies on constructing confidence sets that include the true environment with high probability to construct corresponding policies, one example of which is the Upper Confidence Bound (UCB) approach \cite{tossou2019nearoptimaloptimisticreinforcementlearning, ye2024theoretical}. In this paper, however, we focus on the less explored second approach, \emph{Posterior Sampling}, which adopts the Bayesian framework and treats the environment as a random variable. \textcolor{blue}{The Bayesian reinforcement learning is a natural framework for studying in-context RL with LLMs (Transformer-based policies)}.
One classical posterior sampling algorithm is Thompson Sampling (TS), which has been proved to be sample-efficient and enjoy sublinear Bayesian regret upper bounds  in both RL \cite{moradipari2023improved} and RLHF settings~\cite{wu2023making}.

Apart from TS, \emph{information-directed sampling} (IDS) emerges as a novel and principled online decision-making approach. By incorporating a mutual information term into the policy selection procedure, IDS manages to further encourage exploration about the unknown environment, thus tackling the exploration-exploitation tradeoff to a certain extent \cite{hao2022regret, russo2014learning}. 
Compared with UCB and TS, IDS is more adept at learning complex information-regret structures, and is more flexible and robust to observation noise \cite{zhang2024provably}. In addition, empirical evidence has demonstrated that IDS performs exceptionally well across a range of scenarios, such as sparse linear bandits \cite{hao2021information}, bandits with graph feedback \cite{hao2022contextual}, Markov Decision Processes (MDPs) \cite{hao2022regret}.

Despite their theoretical and empirical advantages, existing IDS-based algorithms are restricted to RL problems with explicitly observable rewards, and are not applicable to RLHF settings. In the LLM era, there is a pressing need for sample-efficient RLHF algorithms,  particularly for scenarios with  large state spaces. 
To tackle these challenges, we first introduce the concept of \emph{surrogate environment}, a compressed (simplified) representation of the potentially complex environment, which helps address the issue of large state spaces. Building on this, and inspired by rate-distortion theory, we design IDS-based RLHF algorithms that are not only theoretically sample-efficient but also computationally easy to implement.

\textbf{Main contributions:} The contribution of this paper can be summarized as follows.

\begin{enumerate}
    \item We first introduce a basic IDS-based algorithm for the RLHF setting where the reward is unobservable and only preference feedback is available (see Sec. \ref{sec:basic_alg}). In each episode, it follows the Bayesian posterior sampling paradigm, and solves an optimization problem that maximizes the sum of an expected value term (exploitation) and a mutual information term (exploration). Here, the mutual information quantifies the amount of information about a learning target (e.g., the environment) that can be gained through the trajectrories and preference.

    \item To tackle the challenge posed by large state spaces, we construct a simplified surrogate environment as the learning target in our algorithm. Using tools from information theory and posterior consistency theory, we prove that our IDS-based algorithm with surrogate environment (Algorithm~\ref{alg}) achieves a Bayesian regret bound of $O(H^{3/2}\sqrt{\log(K(\epsilon)) T})$, where $H$ is the episode length, $T$ is the number of episode, and $K(\epsilon)$  is related to the \emph{covering  number} of  the  environment. We also specialize our algorithm and results to the tabular RLHF, linear RLHF, and contextual dueling bandit settings, and demonstrate the advantages of our algorithm over existing ones.

    \item \textcolor{blue}{In the literature on information-directed sampling (IDS), to the best of our knowledge, there is no efficient implementation of IDS that learns a surrogate environment, which hinders its practical application.}  We propose an Approximate-IDS algorithm (Algorithm~\ref{alg2}) that is computationally more efficient than Algorithm~\ref{alg} while maintaining nearly the same sample efficiency. The advantage of this algorithm is that it does not need to construct the surrogate environment. This algorithm selects policies using an alternative optimization objective that can be optimized with standard RL techniques, such as PPO \cite{schulman2017proximal}. 
    Furthermore, we note that  the
    design principle of  Algorithm \ref{alg2}  is not only effective for preference-based learning but is also applicable to general RL tasks.
\end{enumerate}

\textbf{Highlights on technical novelty:} In the process of constructing the surrogate environment, we introduce a novel distance measure, the \emph{$\ell_g$-distance}, to quantify the discrepancy between two probability measures (see Eqn.~\eqref{new_dis} in Sec. \ref{surr}). 
\textcolor{blue}{The newly proposed distance measure facilitates the design of our computationally efficient algorithm (Algorithm~\ref{alg2}),  while the KL divergence and $\ell_1$-distance are unable to achieve this. 
However, it simultaneously introduces new challenges for theoretical analysis. We overcome this difficulty by investigating the unique properties of the new distance measure, as elaborated in Appendices~\ref{appendix_A} and~\ref{lg}.
}

\textbf{Comparisons with related works:} 
First, we note that most existing works on RLHF assume deterministic rewards, whereas our work considers a more general framework where both transitions and rewards are stochastic. Among existing RLHF algorithms, the most relevant to ours is the TS-based algorithm by~\cite{wu2023making}. Although in the general setting their algorithm's regret bound is not directly comparable to ours (as theirs depends on the \emph{eluder dimension}, while ours depends on the covering number), we note that in the tabular setting, our bound is superior if we coarsely substitute the dimension $d$ with $SA$ in their linear setting. When comparing with prior works on standard RL, we note that our regret bound\footnote{We say $f(n) = \tilde{O}(g(n))$ if $f(n) = O(g(n)\cdot \text{polylog}(n))$.} $\tilde{O}(H^2\sqrt{SAT})$ is superior to the regret bound  $\tilde{O}(H^2\sqrt{S^2A^2T})$ of the surrogate-IDS algorithm by \cite{hao2022regret}, even though we consider a more challenging RLHF setting where we rely only on human feedback to learn the reward model. \textcolor{blue}{Moreover, compared to a prior work on TS for standard RL~\cite{moradipari2023improved}, our analysis method removes a technical assumption that almost all optimal policies visit almost all state action
pairs.}  

\section{Related Works}

\textbf{Reinforcement Learning from Human Feedback (RLHF)}: RLHF has emerged as a critical approach in aligning AI systems with human values, especially in complex tasks where human feedback plays a crucial role \cite{achiam2023gpt,touvron2023llama}. The RLHF framework typically involves a three-stage process: supervised fine-tuning (SFT), reward modeling (RM), and reinforcement learning (RL) using algorithms like Proximal Policy Optimization (PPO)\cite{ouyang2022training,ziegler2019fine}.
Direct Preference Optimization (DPO)~\cite{rafailov2024direct}  is another approach that directly uses generative models as reward models and trains them using preference data. 

The practical success of RLHF has also sparked a variety of theoretical studies. According to the type of preference feedback, these works can be roughly divided into two categories: \emph{action preference}~\cite{furnkranz2012preference,saha2021optimal,ji2024reinforcement,sekhari2024contextual,li2024feel,bai2025online} and \emph{trajectory preference}~\cite{busa2014preference,xu2020preference,pacchiano2021dueling,chen2022human,taranovic2022adversarial,wu2023making}. The literature  on action preferences is generally referred to as the \emph{contextual dueling bandits}.
In this paper, we focus on the trajectory preference. Most of the existing work in this area follow the OFU principle 
with the exception of~\cite{wu2023making,li2024feel} and~\cite{li2024feel}, who investigate a well-known Bayesian method—TS. Note that \cite{wu2023making} uses trajectory preferences and can be applied to the general function approximation framework while \cite{li2024feel} focuses on contextual dueling bandits. We also use the posterior sampling method, but unlike TS, our method follows the princple of  information-directed sampling.

\textbf{Information-Directed Sampling (IDS):} IDS is a design principle for sequential decision-making problems, which balances
exploration and exploitation by evaluating the information gain from each action or trajectory. Ref.~\cite{russo2014learning} first introduces the IDS principle in the bandit setting. They decompose the Bayesian regret into a information ratio term and a cumulative information gain term, and bound the regret by tools from information theory. Based on their work, many studies use this method to analyze the regret of the TS algorithm in bandit settings \cite{russo2016information,dong2018information,bubeck2020first,liu2018information,kirschner2021asymptotically,hao2021information,hao2022contextual}. 

Recently,  \cite{hao2022regret,moradipari2023improved} study the Bayesian regret of IDS and TS without any prior assumptions for MDP settings. Ref.~\cite{moradipari2023improved} focuses on analyzing TS in general settings while~\cite{hao2022regret} proposes a regularized-IDS algorithm for tabular and linear settings. Ref.~\cite{zhang2024provably} uses the principle of IDS to design a   set of algorithms for multi-agent reinforcement learning.
They both use the surrogate environment as the learning target to get a sharper bound. However,  implementing the surrogate version of the algorithm is a  challenge. In this paper, we introduce IDS into RLHF for general MDP settings. We propose an easy-to-implement surrogate algorithm and prove that the regret upper bound has the same order as the original version.

\section{Preliminaries} 
\subsection{Notations} 
For any positive integer $n$, we use $[n]$ to denote the set $\{1,2,\ldots, n\}$. For a measurable space $\mathcal{X}$ and a probability measure $\mu$ on it, we let $\Delta(\mathcal{X}, \mu)$ denote the set of all possible probability distributions over $\mathcal{X}$ that are absolutely continuous with respect to $\mu$. When $\mu$ is clear from the context, we use $\Delta(\mathcal{X})$ for brevity. For two probability densities $p,q$ on $\mathcal{X}$, we denote their Kullback-Leibler (KL) divergence $\kl$ as
\begin{equation*}
    \kl(p\Vert q)\triangleq \int_{\mathcal{X}}p(x)\cdot\log\left(\frac{p(x)}{q(x)}\right)\mathrm{d}x.
\end{equation*}
For two random variables $X$ and $Y$, their mutual information $\I(X;Y)$ is defined as
\begin{equation*}
    \I(X;Y)\triangleq \kl(\mathbb{P}((X,Y) \in \cdot \ ) \Vert \mathbb{P}(X \in \cdot \ ) \times \mathbb{P}(Y \in \cdot \ )).
\end{equation*}
The conditional mutual information of $X$ and $Y$, given another random variable $Z$, is defined as
\begin{equation*}
    \begin{aligned}
    \I(X;Y|Z) \triangleq \mathbb{E}_Z[ \kl(\mathbb{P}((X,Y) \in \cdot \  |Z) 
     \Vert \mathbb{P}(X \in \cdot \ |Z) \times \mathbb{P}(Y \in \cdot \  |Z))].
    \end{aligned}
\end{equation*}


\subsection{Finite-horizon MDPs}

The environment is denoted as $ \mathcal{E}=(\mathcal{S},\mathcal{A},H,\{P_h\}_{h=1}^H,\{R_h\}_{h=1}^H)$, where $\mathcal{S}$ and $\mathcal{A}$ are the measurable state and action spaces respectively, and $H$ is the episode length. For each step $h\in[H]$, $P_h:\mathcal{S} \times \mathcal{A}\to \Delta\left({\mathcal{S}}, \mu_{\mathcal{S}}\right)$ is the transition probability kernel, where $\mu_{\mathcal{S}}$ is the base probability measure on $\mathcal{S}$; $R_h:\mathcal{S}\times \mathcal{A} \to \Delta\left([0,1], \mathrm{Lebesgue}\right)$ is the reward function. Since we mostly deal with the mean value of the reward, we define $r_h(s,a)\triangleq \mathbb{E}_x\left[R_h(x|s,a)\right]=\int_0^1xR_h(x|s,a)\mathrm{d}x$. We assume that $\mathcal{S}, \mathcal{A}$ are known while the transition kernels $\{P_h\}_{h=1}^H$ and rewards $\{R_h\}_{h=1}^H$ are unknown and random.

We consider a Bayesian framework, where we treat the environment $\mathcal{E}$ as a random variable and have a prior belief on $\mathcal{E}$. For each step $h\in[H]$,
let $\Theta_h^{P}$ and $\Theta_h^R$ be the function spaces of $P_h$ and $R_h$ respectively, and let $\Theta_h\triangleq\Theta_h^P\times \Theta_h^R$. The spaces $\Theta_h^{P}$ and $\Theta_h^R$ are assumed to be equipped with prior probability measures, denoted as $\rho_h^P$ and $\rho_h^R$ respectively. Define the full function spaces
$
    \Theta^P\triangleq\prod_{h=1}^H \Theta_h^P, \ \Theta^R\triangleq\prod_{h=1}^H \Theta_h^R,\ \Theta\triangleq\prod_{h=1}^H \Theta_h
$,
which parameterize the set of all environments and also induce the product prior probability measure $\rho^P\triangleq\prod_{h=1}^H \rho_h^P$ for $\Theta^P$, $\rho^R\triangleq\prod_{h=1}^H \rho_h^R$ for $\Theta^R$, and $\rho \triangleq \rho^P\otimes\rho^R$ being the prior of environments. Notice that this setting ensures the independence of the priors over different layers.
Since the notion of the convex combination of environments will be used in our analysis, without loss of generality, we assume $\Theta$ is convex.  

\subsection{Interaction protocol}
The process of an agent interacting with a finite-horizon MDP is as follows. 
The agent starts at an initial state $s_1^t$, which is assumed to be fixed for all episodes $t\in[T]$. In each episode $t\in[T]$, the agent selects two policies $(\pi_0^t,\pi_1^t)$ from the set of all possible policies $\Pi$, where a policy $\pi$ is denoted by stochastic maps $(\pi_1,\ldots,\pi_H)$ with each $\pi_h:\mathcal{S}\to\Delta(\mathcal{A})$. Note that by this definition we assume the policy to be stationary, i.e., depends  only on the current state and layer. At layer $h$ in episode~$t$, for $i=0,1$, the agent observes state pair $(s_h^{t,0},s_h^{t,1})$, separately executes $\pi_i^t$ on $s_h^{t,i}$ to obtain action pair $(a_h^{t,0},a_h^{t,1})$ with probability $\pi_i^t(a_h^{t,i}|s_h^{t,i})$, takes the actions and changes to the next random state $s_{h+1}^{t,i}$ with probability $P_h(s_{h+1}^{t,i}|s_{h}^{t,i},a_{h}^{t,i})$. At state $s_{H+1}$, the agent stops acting and obtains two trajectories $\tau_0^t$ and $\tau_1^t$, where $$\tau_i^t \triangleq (s_1^{t,i},a_1^{t,i},...,s_H^{t,i},a_H^{t,i}).$$ 
In the  RLHF setting, the agent cannot directly receive a numerical reward, but only receives a \emph{preference signal} $o_t$ over trajectory pair $(\tau_0^t,\tau_1^t)$, where $o_t$ is a Bernoulli random variable with $\mathbb{P}(o_t=1|\tau_0^t,\tau_1^t)\triangleq\mathbb{P}(\tau_1^t \text{ is preferred to } \tau_0^t)$. We assume the preference follows the \emph{Bradley-Terry (BT) model} \cite{bradley1952rank}, which has been widely used in existing works on RLHF. The BT model assumes the probability of humans preferring one choice to the other is proportional to the exponential of the value of cumulative reward:
\begin{equation*}
    \mathbb{P}(o_t=1|\tau_0^t,\tau_1^t)
    =\sigma(r(\tau_1^t)-r(\tau_0^t)),
\end{equation*}
where $r(\tau^t)\triangleq\sum_{h=1}^{H}r_h(s_h^t,a_h^t)$ for $\tau^t=(s_1^t,a_1^t,\ldots,s_H^t,a_H^t)$, and $\sigma(x) \triangleq 1/(1+e^{-x})$ is the sigmoid function.

Let $\mathcal{H}_t \triangleq (\tau_0^t, \tau_1^t, o_t)$ be the history of episode $t$ that includes both trajectories and preference feedback, and let $\md_t \triangleq (\mathcal{H}_1,...,\mathcal{H}_{t-1})$ be the entire history up to episode $t$. The history of episode $t$ up to layer $h$ is denoted as
\[
\mathcal{H}_{t,h}\triangleq  (s_1^{t,i}, a_1^{t,i}, \ldots, s_h^{t,i}, a_h^{t,i})_{i \in \{0,1\}}.\]
In the Bayesian setting,  we often need to take conditional expectations with regard to $\md_t$. For brevity, we follow the standard notation in~\cite{hao2022regret}, letting $\dP_t(\cdot)\triangleq\dP(\cdot|\md_t)$, and $\de_t[\cdot]\triangleq\de[\cdot|\md_t]$. \textcolor{blue}{ The mean environment $\bar{\mathcal{E}}_t$ is defined to satisfy $P_h^{\bar{\mathcal{E}}_t}(\cdot|s,a)=\mathbb{E}_t[P_h^{\mathcal{E}}(\cdot |s,a)]$ and $R_h^{\bar{\mathcal{E}}_t}(\cdot|s,a)=\mathbb{E}_t[R_h^{\mathcal{E}}(\cdot |s,a)]$ for all $s\in\mathcal{S}$ and $a\in\mathcal{A}$.} Finally, let $\mathcal{R}_{t,h}\triangleq(r_1^{t,i},...,r_h^{t,i})_{i\in\{0,1\}}$ denote the corresponding potential unobserved rewards, where each $r_h^{t,i}$ is a random variable satisfying $r_h^{t,i}\sim R_h(\cdot|s_h^{t,i}, a_h^{t,i})$.

\subsection{Value function and Bayesian regret}
Define the value function $V_{h,\pi}^{\mathcal{E}}:\mathcal{S}\to[0,H]$ as the expected cumulative rewards received under policy $\pi$ interacting with $\mathcal{E}$ at layer $h$:
\begin{equation*}
    V_{h,\pi}^{\mathcal{E}}(s)\triangleq\mathbb{E}_\pi^{\mathcal{E}}\bigg[  \sum_{h'=h}^{H}r_{h'}(s_{h'},a_{h'})|s_h=s  \bigg],
\end{equation*}
where $\mathbb{E}_\pi^{\mathcal{E}}$ denotes the expectation over the trajectory generated under policy $\pi$ and environment $\mathcal{E}$. We set $V_{H+1,\pi}^{\mathcal{E}}(\cdot)\triangleq 0$. For environment $\E$, let $\pi_\E^*$ be the optimal policy that satisfies $\pi_\E^* = \max_{\pi} V_{h,\pi}^{\mathcal{E}}(s)$ for all $s \in \mathcal{S}$ and $h \in [H]$. Note that under Bayesian settings, $\pi_\E^{*}$ is a function of $\mathcal{E}$, which is also a random variable. 

Finally, for a sequence of policies $\pi = (\pi_t)_{t \in [T]}$ over $T$ episodes, we define the \emph{regret} of $\pi$ in environment $\E$ as
\begin{equation}
R_T(\mathcal{E},\pi)\triangleq\sum_{t=1}^{T} V_{1,\pi_{\mathcal{E}}^{*}}^{\mathcal{E}}(s_1^t)-V_{1,\pi^{t}}^{\mathcal{E}}(s_1^t).
\end{equation}
Since this work focuses on the Bayesian setting, we also define the \emph{Bayesian regret}  as 
\begin{equation}
\label{regret_def}
    BR_T(\pi)\triangleq\mathbb{E}_{\E\sim\rho}[R_T(\mathcal{E},\pi)].
\end{equation}
The task of finding a policy $\pi$ with minimal Bayesian regret, in the context of a finite-horizon MDP, is called a Bayesian RLHF problem.

\section{The basic IDS Algorithm}
This section introduces a basic IDS algorithm for RLHF settings. In Sec. \ref{sec:basic_alg}, we present the generic form of our algorithm with an abstract learning target. Sec. \ref{surr} suggests constructing a discrete surrogate environment as the learning target and then describes an IDS algorithm with the surrogate environment (Algorithm~\ref{alg}). Sec. \ref{regret_analysis} provides the Bayesian regret bound for Algorithm~\ref{alg}, while Sec. \ref{application} specializes this result to tabular RLHF, linear RLHF, and contextual dueling bandits.

\subsection{Algorithm description: a generic form}
\label{sec:basic_alg}

At the beginning of episode $t$, based on the prior distribution $\rho$ and history data $\mathcal{D}_t$, the agent first computes the posterior distribution of the environment $\E\sim\dP(\cdot|\md_t)$, or equivalently, the transition $P$ and reward $R$. Then, the agent chooses a stochastic policy $\pi_{\text{IDS}}^t$ by maximizing a weighted
sum of an expected value term and a mutual information term:
\begin{equation}\label{eq4_1}
    \ids^t=\arg\max_{\pi\in\Pi}\mathbb{E}_t[V_{1,\pi}^{\mathcal{E}}(s_1)]+ \frac{\lambda}{2}\cdot\mathbb{I}_t^\pi\left(\chi;(\mathcal{H}_t,\mathcal{R}_{t,H})\right),
\end{equation}
where $\lambda>0$ is a tunable parameter. Here, $\chi$ is called the \emph{learning target}, which is a random variable and is usually selected as the whole environment $\E$ when the state space is not too large. However, in Sec.~\ref{surr} where we consider large state space cases, we will construct a surrogate environment as the learning target to achieve tighter regret bounds. 

The subscript $t$ in $\mathbb{I}_t^\pi\left(\chi;(\mathcal{H}_t,\mathcal{R}_{t,H})\right)$ in Eqn.~\eqref{eq4_1} means that the distributions of $\chi$ and $(\mathcal{H}_t,\mathcal{R}_{t,H})$ are both conditioned on $\mathcal{D}_t$, and the superscript $\pi$ means
that $(\mathcal{H}_t,\mathcal{R}_{t,H})$ are obtained by executing the policy $\pi$. Intuitively, a larger value of $\mathbb{I}_t^\pi\left(\chi;(\mathcal{H}_t,\mathcal{R}_{t,H})\right)$ indicates that the data obtained at episode $t$ contains more information about the learning target $\chi$. Accordingly, the introduction of mutual information in the policy selection procedure further encourages exploration about the unknown environment, while the expected value term $\mathbb{E}_t[V_{1,\pi}^{\mathcal{E}}(s_1)]$ promotes exploitation. In this way, our algorithm manages to tackle the exploration-exploitation tradeoff to a certain extent.

\subsection{Constructing surrogate environments as learning targets}\label{surr}
In real-world scenarios, the environment is often too complex to be fully included as the agent's learning target $\chi$, thus it is better for the agent to focus only on the significant parts of the environment. In this subsection, we construct a discrete surrogate environment and propose an IDS algorithm with this surrogate environment as the learning target. 

\vspace{5pt}
\subsubsection{A new distance measure}
The discrete environment is constructed using a  covering argument with suitable distance measures. 
Unlike previous works on standard RL settings~\cite{hao2022regret,moradipari2023improved} that use either $\ell_1$-distance or KL-divergence, we propose a new distance measure between two probability measures, called the \emph{$\ell_g$-distance},  which is better suited to our RLHF framework:
\begin{equation}
\label{new_dis}
  \begin{aligned}
    \ell_g(P,Q) 
     \triangleq \sup_{o\in\mathcal{O}} \Vert \log P(\cdot|o) - \log Q(\cdot|o) \Vert_1 
     = \sup_{o\in\mathcal{O}} \int_{x\in\mathcal{X}} \left|\log\frac{P(x|o)}{Q(x|o)}\right| \mathrm{d}\mu_{\mathcal{X}},
  \end{aligned}
\end{equation}
where $\mathcal{O}=\mathcal{S} \times \mathcal{A}$.

\begin{remark}
    For any two vector-valued maps $P,Q$, we define 
    \begin{equation}
        \ell_g(P,Q)\triangleq \sup_{o\in\mathcal{O}}\int_{x\in\mathcal{X}} \sum_{i} \left| \log\frac{P_i(x|o)}{Q_i(x|o)} \right| \mathrm{d}\mu_{\mathcal{X} }
    \end{equation}
    where $P_i$ and $Q_i$ are the $i$-th component of $P$ and $Q$ respectively. This generalization of one-dimension case is useful for the analysis of linear RLHF problems (Theorem \ref{corollary2}).
\end{remark}
\begin{remark}
\textcolor{blue}{To guarantee $\ell_g$ is well-defined, we let $\log \frac{0}{0} \triangleq 0$.} 
    Similar to the KL divergence, we allow for taking infinite values of $\ell_g$, e.g., if there exists a subset $\mathcal{X}'\subset\mathcal{X}$ with positive measure such that $Q(x|o)=0$ but $P(x|o)$ is nonzero on $\mathcal{X}'$, by definition we have $\ell_g(P,Q) =\infty$.
\end{remark}
Although similar to the KL divergence, one of the fundamental properties of $\ell_g$ is that $\ell_g$ is a distance metric, which is more convenient for analysis. 
\begin{lemma}\label{metric}
    $\ell_g$ is a distance metric.
\end{lemma}
\begin{proof}
    By definition, it is easy to see that $\ell_g(P,Q)=\ell_g(Q,P)$ and $\ell_g(P,Q)=0\Leftrightarrow P=Q$. It then suffices to show the triangle inequality. For any three probability distributions $P,Q,R$, we have
    \begin{align*}\label{triangle}
    \ell_g(P,Q)= \sup_{o} \int_{x\in\mathcal{X}} \bigg|\log \frac{P(x|o)}{Q(x|o)}\bigg| 
    &= \sup_{o} \int_{x\in\mathcal{X}} \bigg|\log \frac{P(x|o)}{R(x|o)}-\log\frac{Q(x|o)}{R(x|o)}\bigg| \\
    &\leq \sup_{o} \int_{x\in\mathcal{X}} \bigg|\log \frac{P(x|o)}{R(x|o)}\bigg| + \int_{x\in\mathcal{X}} \bigg|\log \frac{Q(x|o)}{R(x|o)}\bigg| \\
    &= \ell_g(P,R)+\ell_g(Q,R),
    \end{align*}
    which completes the proof of Lemma \ref{metric}.
\end{proof}

To guarantee the existence of a finite coverage, we need  the following assumptions:
\begin{assumption}
\label{assumption2}
$(\Theta,\tau_{\ell_g})$ is a compact topological space, where $\tau_{\ell_g}$ is the topology generated by the metric~$\ell_g$.
\end{assumption}

\begin{assumption}
\label{assumption1}
    For any $P \in \Theta$,  there exist $\beta, B>0$ such that 
    \[\beta \leq \inf_{o,x} \{ P(x|o):P(x|o)\neq 0\} \leq \sup_{o,x} \{ P(x|o)\} \leq B. \]
\end{assumption}

\textcolor{blue}{Note that, our assumption permits the probability density to be zero, but restricts the non-zero support to have a lower bound of $\beta$. This lower bound $\beta$ is only required for Algorithm~\ref{alg2} in Section~\ref{sec:app}; Algorithm~\ref{alg} does not rely on this assumption. In the regret upper bound of Algorithm ~\ref{alg2}, the term involving $\beta$ is $\log \frac{1}{\beta}$. Consequently, $\beta$ can be chosen to be extremely small. For instance, if $\beta = e^{-100}$ (a value far beyond the floating-point precision of modern computers), then $\log \frac{1}{\beta} = 100$. Even in this case, the upper bound is only affected by a constant factor relative to the original bound.  }

Given the new distance $\ell_g$, we introduce the definition of  $\epsilon$-covering number.
\begin{definition}[$\epsilon$-covering number]
For a set $\mathcal{G}$, the $\epsilon$-covering number of $\mathcal{G}$ with respect to $\ell_g$ is  the size $K(\mathcal{G},\epsilon)$ of the smallest set $\{G_1,...,G_{K(\mathcal{G},\epsilon)} \} \subset \mathcal{G}$ such that 
    \begin{equation}
        \forall P \in \mathcal{G}, \exists P' \in \{G_1,...,G_{K(\mathcal{G},\epsilon)} \} : \ell_g(P,P') \leq \epsilon.
    \end{equation}
\end{definition}

\vspace{10pt}
\subsubsection{Partition of the environment} First, we  introduce the concept of $\epsilon$-\emph{value partition},  which must exist based on Assumption~\ref{assumption2}.
\begin{definition}[$\epsilon$-value partition]\label{epsilon_value}
Given any $\epsilon>0$, we say a partition $\{\Theta_k^{\epsilon}\}_{k=1}^K$ over $\Theta$ is an $\epsilon$-value partition for a RLHF problem if for any $k\in[K]$ and $\E,\E'\in\Theta_k^{\epsilon}$,
\begin{equation}
    V_{1,\pi_{\mathcal{E}}^{*}}^{\mathcal{E}}(s_1)-V_{1,\pi_{\mathcal{E}}^{*}}^{\mathcal{E}'}(s_1) \leq \epsilon.
\end{equation}
\end{definition}

We now provide a concrete construction of the $\epsilon$-value partition as follows.
For any $\E_0 \in \Theta$, we define the 
$\epsilon$-ball  centered at $\E_0$ as 
\begin{equation}
\label{ball}
B(\E_0,\epsilon)\triangleq\{\E\in\Theta:\ell_g(\E,\E_0) \leq \epsilon\}.
\end{equation}

Let $\delta_P \triangleq \epsilon/6BH^2$ and $\delta_R \triangleq \epsilon/6BH$. Let $K(\Theta_h^P,\delta_P)$ and $ K(\Theta_h^R,\delta_R) $ be the $\delta_P$-covering and $\delta_R$-covering numbers of $\Theta_h^{P}$ and   $\Theta_h^{R}$ respectively. We denote $\{B_{h}^P(i,\delta_P)\}_{i=1}^{K(\Theta_h^P,\delta_P)}$ and $\{B_{h}^R(j,\delta_R)\}_{j=1}^{K(\Theta_h^R,\delta_R)}$ as the corresponding  $\epsilon$-balls that cover  $\Theta_h^{P}$ and $\Theta_h^{R}$.
For each $i_h \in [K(\Theta_h^P,\delta_P) ]$ and $j_h \in [K(\Theta_h^R,\delta_R)]$, we define 
\begin{equation}
\label{cover}
    \Theta^\epsilon_{h,i_h,j_h}\triangleq\left\{\E\in\Theta \mid P_h^{\mathcal{E}} \in B_{h}^P(i_h,\delta_P),R_h^{\mathcal{E}} \in B_{h}^R(j_h,\delta_R)\right\}.
\end{equation}

Setting $K(\epsilon) \triangleq \prod_{h=1}^H K(\Theta_h^P,\delta_P) \times K(\Theta_h^R,\delta_R)$, we can then find a bijective mapping from $(h,i_h,j_h)$ to $[K(\epsilon)]$, and we obtain an $\epsilon$-value  partition that satisfies $\cup_{k=1}^{K(\epsilon)} \Theta_k^{\epsilon} = \Theta$.\footnote{If an environment $\mathcal{E} \in \Theta$ belongs to more than one partition, we will ensure it only appears in a single partition by truncating the other partitions.} Now, we prove that for any $\E,\E'$ belonging to the same partition,  $$V_{1,\pi_{\mathcal{E}}^{*}}^{\mathcal{E}}(s_1)-V_{1,\pi_{\mathcal{E}}^{*}}^{\mathcal{E}'}(s_1) \leq \epsilon.$$

By Lemma \ref{lemma}, we have
    \begin{align}\label{A5}
        &V_{1,\pi_{\mathcal{E}}^{*}}^{\E}(s_1)-V_{1,\pi_{\mathcal{E}}^{*}}^{\E'}(s_1) \nonumber\\
        &=\sum_{h=1}^{H}\mathbb{E}_{\pi_{\E}^{*}}^{\E'}\bigg[\mathbb{E}_{s'\sim P_h^{\E}(\cdot|s_h,a_h)}\left[ V_{h+1,\pi_{\mathcal{E}}^{*}}^{\E}(s')\right] -\mathbb{E}_{s'\sim P_h^{\E'}(\cdot|s_h,a_h)}\left[ V_{h+1,\pi_{\mathcal{E}}^{*}}^{\E}(s')\right]\bigg]+\sum_{h=1}^H\mathbb{E}_{\pi_{\E}^{*}}^{\E'}\left[ R_h^\E(s_h,a_h)-R_h^{\E'}(s_h,a_h)\right]\nonumber\\
        &\leq\sum_{h=1}^H\mathbb{E}_{\pi_{\E}^{*}}^{\E'}\bigg[\int_{\mathcal{S}}\left|P_h^\E(s'|s_h,a_h)-P_h^{\E'}(s'|s_h,a_h)\right|\cdot V_{h+1,\pi_{\E}^{*}}^{\E}(s')\mathrm{d}\mu_{\mathcal{S}}+\int_{[0,1]}\left|x\left(R_h^\E(x|s_h,a_h)-R_h^{\E'}(x|s_h,a_h)\right)\right|\mathrm{d}x\bigg]\nonumber\\
        &\leq \sum_{h=1}^H\mathbb{E}_{\pi_{\E}^{*}}^{\E'}\bigg[HB\cdot\int_{\mathcal{S}}\left|\log\frac{P_h^\E(s'|s_h,a_h)}{P_h^{\E'}(s'|s_h,a_h)}\right|\mathrm{d}\mu_{\mathcal{S}}+B\cdot\int_{[0,1]}\left|\log\frac{R_h^\E(x|s_h,a_h)}{R_h^{\E'}(x|s_h,a_h)}\right|\mathrm{d}x\bigg]\nonumber\\
        &\leq\sum_{h=1}^H\mathbb{E}_{\pi_{\E}^{*}}^{\E'}\left[HB\cdot2\delta_P+B\cdot2\delta_R\right]=\frac{2\epsilon}{3}\leq\epsilon.
    \end{align}
    where the second inequality is due to the fact that $V_{h+1,\pi_{\mathcal{E}}^{*}}^{\E}(s')\leq H$, and $|a-b| \leq B\cdot|\log\frac{a}{b}|$ for any $a,b\in(0,B)$. The last inequality is due to the definition of $\ell_g$: since $\E,\E'$ lie in the same $\Theta_k^\epsilon$, we have $\ell_g(P_h^{\E},P_h^{\E'})\leq 2\delta_P$ and $\ell_g(R_h^{\E},R_h^{\E'})\leq 2\delta_R$.  This shows that $\{\Theta_k^\epsilon\}_{k=1}^K$ gives an $\epsilon$-value partition.

\begin{remark}
  \textcolor{blue}{
An important distinction between the  $\ell_g$-distance and the $\ell_1$-distance is that the $\epsilon$-ball under the $\ell_g$-distance is not  convex. 
In other words, there exist instances of probability measures for which the $\epsilon$-ball defined in Eq.~\eqref{ball} is non-convex. We provide a counterexample in the Appendix~\ref{lg} to demonstrate this non-convexity. While the lack of convexity does not affect the partitioning of the environment, it does influence the construction of the surrogate environment in the subsequent analysis.
}  
\end{remark}

\subsubsection{Construct the surrogate environment}
Based on the above $\epsilon$-value partition, we
explicitly construct the \emph{surrogate environment} $\tilde{\mathcal{E}}_t^{*} $ 
 for episode $t$ as:
\begin{equation}
\label{lemma_tmp1}
    \tilde{\mathcal{E}}_t^{*} = \tilde{\mathcal{E}}_{k,t}^{*}  \text{ iff } \mathcal{E} \in \Theta_k^{\epsilon},
\end{equation}
where $\tilde{\mathcal{E}}_{k,t}^{*} \triangleq \mathbb{E}_t \left[ \mathcal{E}|\mathcal{E} \in \Theta_k^{\epsilon}\right]$.  Since $\cup_{k=1}^{K(\epsilon)} \Theta_k^{\epsilon} = \Theta$, for any $\mathcal{E}$, there exists $k \in [K(\epsilon)]$ such that $\mathcal{E} \in \Theta_k^{\epsilon}$. Hence, the surrogate environment is well defined.
For this surrogate environment, we have the following result.

\begin{lemma}
\label{partition}  Fix $t\in[T]$ and environment $\E\in\Theta$. 
Given the $\epsilon$-value partition $\{\Theta_k^{\epsilon}\}_{k=1}^{K(\epsilon)}$ (Eqn.~\eqref{cover}), the  surrogate environment $\tilde{\mathcal{E}}_t^{*}$ constructed by Eqn.~\eqref{lemma_tmp1} satisfies the following:
    
\begin{enumerate}
    \item For any  $(s,a,h)\in\mathcal{S}\times\mathcal{A} \times [H]$ and any instance $(\tilde{\mathcal{E}}_t^{*},\mathcal{E} ) \sim \mathbb{P}_t(\tilde{\mathcal{E}}_t^{*},\mathcal{E})$ , it holds that
    \begin{equation}\label{4_3} 
    \ell_g(P_h^{\tilde{\mathcal{E}}_t^{*}}, P_h^{\mathcal{E}} )\leq\frac{\epsilon}{2BH^2}, \quad  \ell_g(R_h^{\tilde{\mathcal{E}}_t^{*}}, R_h^{\mathcal{E}}) \leq \frac{\epsilon}{2BH}
    \end{equation}
\item  The following inequality holds:
  \begin{equation}
  \mathbb{E}_t\big[V_{1,\pi_{\mathcal{E}}^{*}}^{\mathcal{E}}(s_1^t)-V_{1,\pi_{\text{TS}}^{t}}^{\mathcal{E}}(s_1^t)\big] - \mathbb{E}_t \big[V_{1,\pi_{\mathcal{E}}^{*}}^{\tilde{\mathcal{E}}_t^{*}}(s_1^t)-V_{1,\pi_{\text{TS}}^{t}}^{\tilde{\mathcal{E}}_t^{*}}(s_1^t) \big] \leq \epsilon,
  \end{equation}
  where $\ts^t\triangleq\arg\max_{\pi\in\Pi}V_{1,\pi}^\E(s_1^t)$ is the TS policy that depends on the random environment $\mathcal{E}$.
\end{enumerate}  
\end{lemma}

\textcolor{blue}{A necessary step in the  regret analysis is to bound $\ell_g(P_h^{\tilde{\mathcal{E}}_t^{*}}, P_h^{\mathcal{E}} )$ by the radius of the $\epsilon$-balls, i.e., Eq.~\eqref{4_3}. Since the $\epsilon$-ball under $\ell_1$ is convex, the posterior mean  $\tilde{\mathcal{E}}_{k,t}^{*}$  also lies in $\Theta_k^{\epsilon} $. We can  directly derive that $ \ell_1(P_h^{\tilde{\mathcal{E}}_t^{*}}, P_h^{\mathcal{E}} ) \leq 2 \epsilon. $ However, $\Theta_k^{\epsilon}$  is not  convex under the new metric $\ell_g$, hence $\tilde{\mathcal{E}}_t^{*}$ and $\mathcal{E}$ may not lie in the same partition.  Therefore, we are unable to immediately obtain Eq.~\eqref{4_3}. Although we cannot  use the convexity of  $\epsilon$-balls under $\ell_g$, a specific geometric property of $\ell_g$ (Lemma \ref{center})  significantly simplifies our proof.}

\begin{proof}
(1) 
    Based on our environment partitioning, $\Theta_k^{\epsilon}$ is a ball. 
    Let $\mathcal{C}$ be the center of $\Theta_k^{\epsilon}$, by Lemma~\ref{center}, we have $\ell_g(P_h^{\tilde{\mathcal{E}}_t^{*}},P_h^{\mathcal{C}})\leq 2\delta_P$. Then, by triangle inequality of $\ell_g$ (Lemma \ref{metric}), we have $$\ell_g(P_h^{\tilde{\mathcal{E}}_t^{*}},P_h^{\mathcal{E}}) \leq \ell_g(P_h^{\tilde{\mathcal{E}}_t^{*}},P_h^{\mathcal{C}}) + \ell_g(P_h^{\mathcal{E}},P_h^{\mathcal{C}})\leq 3\delta_P=\frac{\epsilon}{2BH^2}.$$
The analysis for the reward term $\ell_g(R_h^{\tilde{\mathcal{E}}_t^{*}},R_h^{\mathcal{E}})$ is exactly the same as above, which yields the proof of the first conclusion in Lemma \ref{partition}.

\vspace{5pt}
(2)   For the second property, we divide $\mathbb{E}_t\left[V_{1,\pi_{\mathcal{E}}^{*}}^{\mathcal{E}}(s_1^t)-V_{1,\ts^{t}}^{\mathcal{E}}(s_1^t)\right] - \mathbb{E}_t \left[V_{1,\pi_{\mathcal{E}}^{*}}^{\tilde{\mathcal{E}}_t^{*}}(s_1^t)-V_{1,\ts^{t}}^{\tilde{\mathcal{E}}_t^{*}}(s_1^t) \right]$ into two parts.
    \begin{itemize}
        \item We first show that $\de_t\left[V_{1,\ts^{t}}^{\mathcal{E}}(s_1^t)\right]=\de_t\left[V_{1,\ts^{t}}^{\tilde{\mathcal{E}}_t^{*}}(s_1^t)\right].$ 
        Let $\mathcal{E}_t \sim \mathbb{P}(\cdot | \mathcal{D}_t)$ be an independent sample of $\mathcal{E}$. By the law of total expectation and the definition of $\tilde{\mathcal{E}}_{t}^{*}$, we have
        \begin{align*}
        \mathbb{E}_t\left[ V_{1,\ts^{t}}^{\tilde{\mathcal{E}}_{t}^{*}}(s_1^t)  \right] 
       &= \sum_{k=1}^{K}\dP(\E \in \Theta_k^{\epsilon})\cdot \de_t\left[ V_{1,\ts^{t}}^{\tilde{\E}_{t}^{*}}(s_1^t) \middle| \mathcal{E} \in \Theta_k^{\epsilon}  \right] \\
       &= \sum_{k=1}^{K}\dP(\E \in \Theta_k^{\epsilon})\cdot \de_t\left[ V_{1,\ts^{t}}^{\tilde{\E}_{k,t}^{*}}(s_1^t)  \right].
       \end{align*}
      By the definition of $\tilde{\E}_{k,t}^{*} $, 
      \[ \de_t\left[ V_{1,\ts^{t}}^{\tilde{\E}_{k,t}^{*}}(s_1^t)  \right] = \int_{\E'\in\Theta_k^\epsilon}\de_t\left[ V_{1,\ts^{t}}^{\E'}(s_1^t)  \right]\mathrm{d}\dP(\E_t=\E'|\E_t\in\Theta_k^\epsilon). \]
    Then, we have
    \begin{align*}
        \mathbb{E}_t\left[ V_{1,\ts^{t}}^{\tilde{\mathcal{E}}_{t}^{*}}(s_1^t)  \right] 
       &= \sum_{k=1}^{K}\dP(\E \in \Theta_k^{\epsilon})\cdot \de_t\left[ V_{1,\ts^{t}}^{\tilde{\E}_{k,t}^{*}}(s_1^t)  \right]\\
       &=\sum_{k=1}^{K}\dP(\E \in \Theta_k^{\epsilon})\cdot \int_{\E'\in\Theta_k^\epsilon}\de_t\left[ V_{1,\ts^{t}}^{\E'}(s_1^t)  \right]\mathrm{d}\dP(\E_t=\E'|\E_t\in\Theta_k^\epsilon) \\
       & \overset{(a)}{=} \sum_{k=1}^{K}\dP(\E \in \Theta_k^{\epsilon})\cdot \int_{\E'\in\Theta_k^\epsilon}\de_t\left[ V_{1,\ts^{t}}^{\E'}(s_1^t)\middle| \E_t\in\Theta_k^\epsilon \right]\mathrm{d}\dP(\E_t=\E'|\E_t\in\Theta_k^\epsilon) \\
       &= \sum_{k=1}^{K}\dP(\E \in \Theta_k^{\epsilon})\cdot \de_t\left[ V_{1,\ts^{t}}^{\E_t}(s_1^t) \middle| \E_t \in \Theta_k^{\epsilon}  \right] \\
       &\overset{(b)}{ =} \de_t\left[V_{1,\ts^t}^\E(s_1^t)\right].
    \end{align*}
   where (a) uses the fact that $\ts^t=\pi^{*}_{\E_t'} $, $\E_t'$ is independent of $\E_t$, (b) follows from $\E_t$ is an independent sample of $\E$.

    \item Next, we show that $\de_t\left[V_{1,\pi_\E^*}^{\mathcal{E}}(s_1^t)\right]-\de_t\left[V_{1,\pi_\E^*}^{\tilde{\mathcal{E}}_t^{*}}(s_1^t)\right]\leq \epsilon.$ Adopting the same decomposition  trick  as in Eqn.~\eqref{A5}, we have
    \begin{align}
        &V_{1,\pi_{\mathcal{E}}^{*}}^{\E}(s_1^t)-V_{1,\pi_{\mathcal{E}}^{*}}^{\tilde{\mathcal{E}}_t^{*}}(s_1^t) \nonumber\\
        &\leq \sum_{h=1}^H\mathbb{E}_{\pi_{\E}^{*}}^{\tilde{\mathcal{E}}_t^{*}}\bigg[HB\cdot\int_{\mathcal{S}}\left|\log\frac{P_h^\E(s'|s_h,a_h)}{P_h^{\tilde{\mathcal{E}}_t^{*}}(s'|s_h,a_h)}\right|\mathrm{d}\mu_{\mathcal{S}}+B\cdot\int_{[0,1]}\left|\log\frac{R_h^\E(x|s_h,a_h)}{R_h^{\tilde{\mathcal{E}}_t^{*}}(x|s_h,a_h)}\right|\mathrm{d}x\bigg]\nonumber\\
        &\leq\sum_{h=1}^H\mathbb{E}_{\pi_{\E}^{*}}^{\tilde{\mathcal{E}}_t^{*}}\left[HB\cdot\frac{\epsilon}{2BH^2}+B\cdot\frac{\epsilon}{2BH}\right]=\epsilon,
    \end{align}
    where the second inequality is due to Eqn.~\eqref{4_3}. Adding up the two parts yields the proof of the second property in Lemma~\ref{partition}. By this, we have finished the proof of Lemma~\ref{partition}.
    \end{itemize}
\end{proof}

It is worth noting that Eqn.~\eqref{4_3} represents a unique property of the surrogate environment, specifically attributed to our metric $\ell_g$, which distinguishes it from the KL divergence. It can be proven that the $\ell_1$-distance also possesses this property. However, as seen in Section~\ref{sec:app}, Proposition~\ref{prop} cannot be guaranteed under the $\ell_1$-distance, making it difficult to design efficient approximation algorithms.

\begin{algorithm}[tb]
    \caption{IDS for RLHF}
    \label{alg}
 \begin{algorithmic}[1]
    \STATE {\bfseries Input: } Priors $\rho^P,\rho^R$,baseline policy $\pi_0$, $ \lambda>0$, surrogate environment partition tolerance $\epsilon>0$.  
    \FOR{$t=1$ {\bfseries to} $T$}
    \STATE Compute posteriors:
    \begin{equation}
    \label{alg_tmp1}
        \rho^P_t(P) \propto \rho^P(P) \prod_{i=1}^{t-1}\prod_{h=1}^{H}P_h(s_{h+1}^{i,1}|s_h^{i,1},a_h^{i,1})
    \end{equation}
    \begin{equation}
    \label{alg_tmp2}
    \begin{aligned}
        \rho^R_t(R) \propto \rho^R(R) \prod_{i=1}^{t-1}
        \big(o_i\sigma(r(\tau_1^i)-r(\tau_0^i))
         +(1-o_i)\sigma(r(\tau_0^i)-r(\tau_1^i))\big)
    \end{aligned}
    \end{equation}

    \STATE Compute the surrogate environment $\tilde{\mathcal{E}}_t^{*}$, and update policy by
    \begin{equation*}
    \ids^t=\arg\max_{\pi\in\Pi}\mathbb{E}_t[V_{1,\pi}^{\mathcal{E}}(s_1)]+ \frac{\lambda}{2} \mathbb{I}_t^\pi\left(\tilde{\mathcal{E}}_t^{*};(\mathcal{H}_t,\mathcal{R}_{t,H})\right)
    \end{equation*}
    \STATE Sample $\tau_0^t \sim \pi_0,\tau_1^t \sim \ids^t$.
    \STATE Obtain preference feedback $o_t$ on $\{\tau_0^t,\tau_1^t\}$.
    
    \ENDFOR
    
 \end{algorithmic}
 \end{algorithm}

\vspace{10pt}
\subsubsection{IDS with surrogate environments}
The pseudo-code of our IDS algorithm (with the learning target being the surrogate environment $\tilde{\mathcal{E}}^{*}_t$) is shown in Algorithm~\ref{alg}.
\textcolor{blue}{ The algorithm requires  priors $\rho^P, \rho^R$ as input. Prior refers to the initial assumptions about model parameters before learning begins. Suitable priors can be derived from existing domain knowledge (e.g., robot dynamics parameters, user behavior patterns). For instance, a Gaussian prior might be adopted in robotic ~\cite{haninger2022model,haninger2023model}, while a Dirichlet prior could be used in multi-agent collaboration ~\cite{wu2021too}. In our paper, an appropriate prior can be selected for the IDS algorithm according to practical applications. }
Roughly speaking, the agent, at each episode $t$, first computes the posterior distributions of the transition kernel $P$ and reward function $R$ (as shown in Eqns.~\eqref{alg_tmp1}-\eqref{alg_tmp2}). Then, the agent computes the surrogate environment $\tilde{\E}^*_t$ based on Eqn.~\eqref{lemma_tmp1}, and chooses the policy 
$$ 
\ids^t=\arg\max_{\pi\in\Pi}\mathbb{E}_t[V_{1,\pi}^{\mathcal{E}}(s_1)]+ \frac{\lambda}{2}\cdot\mathbb{I}_t^\pi\left(\tilde{\mathcal{E}}_t^{*};(\mathcal{H}_t,\mathcal{R}_{t,H})\right).
$$
We sample two trajectories from the baseline policy $\pi_0$ and the IDS policy $\ids^t$, respectively, and then obtain a preference $o_t$ regarding the two trajectories. Moreover, human feedback and state-action sequence data  are added to the history data $\mathcal{D}_t$ for updating the posterior distribution for the next episode.


\subsection{Regret analysis of Algorithm~\ref{alg}}
\label{regret_analysis}
Before presenting our main results, we need to first introduce a notion of \emph{value diameter}. For any $\E$, we define the corresponding value diameter  $\alpha_\E$ as
\begin{equation*}
  \begin{aligned}
    \alpha_{\mathcal{E}}
    \triangleq \max_{1\leq h\leq H}\left\{\sup_s V_{h,\pi_{\mathcal{E}}^{*}}^{\mathcal{E}}(s) -\inf_s V_{h,\pi_{\mathcal{E}}^{*}}^{\mathcal{E}}(s) \right\} 
      +\max_{h,s,a}\left\{r_h^{\sup}(s,a)-r_h^{\inf}(s,a)\right\}.
  \end{aligned}
\end{equation*}
Since the reward is bounded by $[0,1]$, we have $\alpha_{\mathcal{E}} \leq H+1$.
The \emph{average value diameter} over $\Theta$
is denoted by $\alpha\triangleq\mathbb{E}_{\mathcal{E}\sim \rho}\big[\alpha_{\mathcal{E}}^2\big]^{1/2}$. 
Similar to the prior work \cite{moradipari2023improved}, we need to make the following assumption about posterior consistency.

\begin{assumption}[Posterior Consistency]
\label{assumption3}
    Under our preference model, the posterior distribution of environment is strongly consistent.
\end{assumption}
This means as the sample size approaches infinity, the posterior distribution of environment obtained through Eqns.~\eqref{alg_tmp1}-\eqref{alg_tmp2} tends to concentrate around the true distribution. In other words, the posterior distribution will correctly identify the true environment that generates these trajectory data.

\begin{theorem}
\label{theorem}
    Given a Bayesian RLHF problem, for any $\epsilon>0$ and sufficiently large $T$, by choosing $\lambda=\sqrt{\alpha^2TH/\log(K(\epsilon))}$, we have
    \begin{equation}
       BR_T(\pi_{\text{IDS}})\leq \alpha \sqrt{TH\log(K(\epsilon))}+T\epsilon+T_0,
    \end{equation}
    where $T_0$ is a fixed positive integer that is independent of $T$. 
    Setting $\epsilon=\frac{1}{T}$, our regret upper bound is of order $$O\left(H^{\frac{3}{2}}\sqrt{T\log (K(\frac{1}{T}))} \right).$$
    \label{thm:1}
\end{theorem}

The detailed proof of Theorem~\ref{thm:1} is deferred to Appendix~\ref{proof_theorem1}.
We point out that the existing TS-based RLHF algorithm \cite{wu2023making} has an upper bound of order $$\tilde{O}(H^{2}\sqrt{T(\ell_{P}+\ell_R)}(\text{dim}_1(P,1/T))+\text{dim}_1(R,1/T)),$$ where $\text{dim}_1(P,1/T))$ and $\text{dim}_1(R,1/T))$ are the $\ell_1$-norm eluder dimension of the  transition and reward function class, $\ell_P$ and $\ell _R$ are the bracketing  covering number  of the  transition and reward function class.   Without considering  the way to characterize the complexity of the the reward and the transition model (i.e., via covering number or eluder dimension), our bound is superior to theirs by a factor of $\sqrt{H}$.


\begin{remark}

    The regret upper bounds in some related works  \cite{saha2023dueling,wu2023making} are related to the derivative bound of the link function. However, our upper bound is independent of the link function we use (which is the sigmoid function). This is because the posterior consistency assumption implicitly imposes requirements on the link function — the link function should be monotonically increasing to ensure that better trajectories correspond to higher preference probabilities. For example, if the link function is equal to a constant $\frac{1}{2}$, then the posterior distribution of rewards would not change (according to the posterior update rule in Eqn.~\eqref{alg_tmp2}), and thus could not converge to the true distribution. Therefore, the assumption in previous work of a strictly positive lower bound on the link function's derivative is encompassed by our posterior consistency assumption.

\end{remark}

\subsection{Applications}
\label{application}
Finally, we show that our algorithm can be applied in multiple scenarios, such as tabular RLHF, linear RLHF, and contextual dueling bandits.
\begin{definition}[Tabular RLHF]
    We say a Bayesian RLHF problem is tabular if $|\mathcal{S}|=S$ and $|\mathcal{A}|=A$ are both finite.
\end{definition}

\begin{definition}[Linear RLHF]
    Let $\phi^{P}:\mathcal{S}\times\mathcal{A} \rightarrow \mathbb{R}^d$ and $ \phi^{R}:\mathcal{S}\times\mathcal{A} \rightarrow \mathbb{R}^d $ be known feature maps with bounded norms $||  \phi^{P}(s,a)||_2 \leq 1$ and $||  \phi^{R}(s,a)||_2 \leq 1$. We say a Bayesian RLHF problem is linear if for any $\mathcal{E}= \{ (P_h^{\mathcal{E}},R_h^{\mathcal{E}}) \}_{h=1}^{H} \in \Theta$, there exists  vector-valued maps $ \psi_h^{P,\mathcal{E}}$ and $\psi_h^{R,\mathcal{E}} $ with bounded $\ell_2$-norm such that for every $(s,a) \in \mathcal{S}\times \mathcal{A}$,
    $$ P_h^{\mathcal{E}}(\cdot|s,a)=\langle \phi^{P}(s,a),\psi_h^{P,\mathcal{E}}(\cdot) \rangle,$$
    $$ R_h^{\mathcal{E}}(\cdot|s,a)= \langle \phi^{R}(s,a),\psi_h^{R,\mathcal{E}}(\cdot) \rangle. $$
    We assume that each component of the vector-valued maps $\psi_h^{P,\mathcal{E}}$ and $\psi_h^{R,\mathcal{E}} $ belongs to some compact set $\mathcal{F} \subset L^2$, i.e.,  $\forall i \in [d]$, $(\psi_h^{P,\mathcal{E}})_i \in \mathcal{F}$ and $(\psi_h^{R,\mathcal{E}})_i \in \mathcal{F}$. 
\end{definition}

Specializing Theorem \ref{theorem} to tabular and linear Bayesian RLHF problems, we have the following Bayesian regret bounds. The proofs of Theorems \ref{corollary1} and \ref{corollary2} are deferred to Appendices \ref{b3} and \ref{b4} respectively. 
\begin{theorem}[Tabular RLHF]
    \label{corollary1}
    Given a tabular Bayesian RLHF problem, for any $\epsilon>0$ and sufficiently large $T$, we have
    \begin{equation*}
       BR_T(\ids)\leq \alpha H\sqrt{3SAT\log\left(\frac{6H^2\sqrt{S}}{\epsilon}\right)}+T\epsilon+T_0,
    \end{equation*}
    where $T_0$ is a fixed integer that is independent of $T$. Setting $\epsilon = \frac{1}{T}$, our regret bound is of order $\tilde{O}(\sqrt{SAH^4T})$.
    
\end{theorem}
Recall that the IDS algorithm proposed for the tabular RL setting~\cite{hao2022regret} has a regret upper bound of order $\tilde{O}(\sqrt{S^2A^2H^4T})$. Compared to their result, our method relies on less informative data (preference feedback instead of directly observable rewards) but achieves a better regret bound by a factor of $S$ and $A$. This improvement is primarily due to our refined analytical techniques, inspired by recent advancements in TS~\cite{moradipari2023improved}.

\begin{theorem}[Linear RLHF]
    \label{corollary2}
     Let $M \triangleq \sup_{i,s} \max \{(\psi_h^{P}(s))_i, (\psi_h^{R}(s))_i\}$ 
    and  $K_{\mathcal{F}}(\epsilon)$ denote the  $\frac{\epsilon}{dMH^2}$-covering number of $\mathcal{F}$. 
      Given a linear RLHF problem, for any $\epsilon>0$ and sufficiently large $T$, we have
    \begin{equation}
       BR_T(\ids)\leq \alpha H \sqrt{d T  \log(K_{\mathcal{F}}(\epsilon))}+T\epsilon+T_0,
    \end{equation}
    where $T_0$ is a fixed integer that is independent of $T$.  Setting $\epsilon=\frac{1}{T}$, this upper bound is of order $$O\left(H^2\sqrt{dT \log (K_{\mathcal{F}}(\frac{1}{T}))}\right).$$
\end{theorem}
Compared to \cite{wu2023making}, which derives a regret upper bound of $\tilde{O}(H^{11/2}d^{17/2}\sqrt{T})$ for their TS algorithm, our regret upper bound is better when the covering number of the linear MDP is not of exponential size. If we convert their result to the tabular setting by coarsely substituting $d$ with $SA$, our regret bound is also better in terms of $H, S, A.$    
However, we also point out that the above comparison is not an apples-to-apples comparison, as we consider Bayesian regret, while they consider frequentist regret, and their algorithm also accounts for the number of queries.

\begin{corollary}[Contextual Dueling Bandits]
    Contextual dueling bandits are a simplified version of our MDP setting (with $H=1$) and have been extensively studied in previous RLHF research~\cite{ye2024theoretical, zhu2023principled, li2024feel}. 
    By setting $H=1$ in Theorem \ref{corollary2}, the Bayesian regret for Algorithm \ref{alg} in the linear contextual dueling bandit problem satisfies
    \begin{equation*}
       BR_T(\pi_{\text{IDS}})\leq 2 \sqrt{dT\log(K_{\mathcal{F}}(\epsilon))}+T\epsilon+T_0,
    \end{equation*}
    for any $\epsilon>0$ and sufficiently large $T$. Setting $\epsilon=\frac{1}{T}$, the regret upper bound is of order $\tilde{O}(\sqrt{dT})$. 

\end{corollary}
Without considering the covering number of linear environment, our regret upper bound is better than  $\tilde{O}(d\sqrt{T})$ derived by~\cite{li2024feel}. Another work~\cite{saha2021optimal} assumes a finite number of arms with a regret upper bound of $\tilde{O}(\sqrt{dT})$, while we assume that the parameter space of the linear MDP is compact.

\section{The Approximate-IDS Algorithm }
\label{sec:app}

While the IDS algorithm (Algorithm \ref{alg}) is principled and sample-efficient, it suffers from relatively high computational complexity. This is because the calculation of the surrogate environment $\tilde{\mathcal{E}}_t^{*}$ (which depends on the construction of $\epsilon$-value partition) is challenging.
As a remedy, we develop a computationally efficient algorithm, named \emph{Approximate-IDS}, whose optimization objective is independent of the surrogate environment (thus avoids the partition of $\Theta$ in computation) and has finer properties for analysis. This allows the algorithm to
be computed efficiently by traditional RL algorithms from standard RL theory. 

\subsection{Algorithm description}
The pseudo-code for Approximate-IDS is shown in Algorithm \ref{alg2}. For convenience of description, we define 
\begin{equation}
    \mathrm{KL}^h_{s,a}(\mathcal{E},\mathcal{E}') \!\triangleq \! \kl \big( 
              (P_h^{\mathcal{E}} \otimes R_h^{\mathcal{E}})(\cdot | s, a) 
           || (P_h^{\mathcal{E}'} \otimes R_h^{\mathcal{E}'})(\cdot | s, a) 
          \big),
\end{equation}
and 
\begin{equation}
\label{rhbar}
  \bar{r}_h(s_h,a_h) \triangleq r_h(s_h,a_h) + \frac{\lambda}{2} \cdot \mathbb{E}_t \big[\mathrm{KL}_{s_h,a_h}^h(\mathcal{E},\bar{\mathcal{E}}_t) \big],  
\end{equation}

where $\bar{\E}_t$ denotes the posterior mean of $\E$ given $\md_t$, i.e., $P_h^{\bar{\mathcal{E}}_t}(\cdot|s,a)=\mathbb{E}_t[P_h^{\mathcal{E}}(\cdot |s,a)]$ and $R_h^{\bar{\mathcal{E}}_t}(\cdot|s,a)=\mathbb{E}_t[R_h^{\mathcal{E}}(\cdot |s,a)]$.

The overall procedure  is similar to that of Algorithm~\ref{alg}, with the key difference being the selection of the IDS policy (Line~4). Intuitively,    $\mathcal{E}$ is sufficiently close to $\tilde{\mathcal{E}}_t^{*}$ under metric $\ell_g$, thus it is reasonable to use $\mathcal{E}$ directly for mutual information computation. Given trajectories and rewards, the additional environmental information revealed by human feedback, i.e., $\mathbb{I}_{t}^{\pi}\big(\tilde{\mathcal{E}}_t^{*}; o_{t} \mid (\mathcal{H}_{t,H},\mathcal{R}_{t,H})\big)$, can be disregarded. 
Thus, we use the entire environment $\mathcal{E}$ instead of the surrogate environment $\tilde{\mathcal{E}}_t^{*}$  to compute the mutual information and discard the information of the trajectory generated by the baseline policy $\pi_0$ and human feedback. Therefore, we replace the mutual information term $\mathbb{I}_t^\pi\big(\tilde{\mathcal{E}}_t^{*};(\mathcal{H}_t,\mathcal{R}_{t,H})\big)$  by $\sum_{h=1}^{H} \mathbb{E}_t \big[ \de_\pi^{\bar{\mathcal{E}}_t} [\mathrm{KL}_{s_h,a_h}^h(\mathcal{E},\bar{\mathcal{E}}_t)] \big]$ (Eqn.~\eqref{temp1} in Lemma \ref{mutual}). We can compute the approximate IDS policy as follows:
\begin{equation}
    \begin{aligned}
        \app^t & \!=\! \arg \max_{\pi \in \Pi} \mathbb{E}_t[V_{1,\pi}^{\mathcal{E}}(s_1)] \!+\! \frac{\lambda}{2} \sum_{h=1}^{H} \mathbb{E}_t \big[ \de_\pi^{\bar{\mathcal{E}}_t} [\mathrm{KL}_{s_h,a_h}^h(\mathcal{E},\bar{\mathcal{E}}_t)] \big]\\
        & \overset{(a)}{=}  \arg \max_{\pi \in \Pi}\de_\pi^{\bar{\mathcal{E}}_t}\left[\sum_{h=1}^{H}\bar{r}_h(s_h,a_h)\right],
    \end{aligned}
\end{equation}
where (a) uses the linearity of expectation and independence of priors over different layers.


\begin{algorithm}[tb]
    \caption{Approximate-IDS for RLHF}
    \label{alg2}
 \begin{algorithmic}[1]
    \STATE {\bfseries Input: } Priors $\rho^P,\rho^r$,baseline policy $\pi_0$, $ \lambda>0$.  
    \FOR{$t=1$ {\bfseries to} $T$}
    \STATE Compute the posterior as Algorithm \ref{alg} (Line 3).

    \STATE  $\app^t=\arg\max_{\pi\in\Pi}\de_\pi^{\bar{\mathcal{E}}_t}\left[\sum_{h=1}^{H}\bar{r}_h(s_h,a_h)\right]
    $

    \STATE Sample $\tau_0^t \sim \pi_0,\tau_1^t \sim \app^t$.
    \STATE Obtain preference feedback $o_t$ on $\{\tau_0^t,\tau_1^t\}$.
    
    \ENDFOR
    
 \end{algorithmic}
 \end{algorithm}

\textcolor{blue}{Note that $\bar{r}$ and $ \bar{\mathcal{E}}_t$ are both independent of the surrogate environment, and can be well approximated by Monte Carlo sampling. Therefore, by introducing  $\bar{r}$, solving $\app^t$ at episode $t$ is equivalent to finding an optimal policy based on MDP $\{P_h^{\bar{\mathcal{E}}_t}, \bar{r}_h\}_{h=1}^H$, which can be solved efficiently by the PPO algorithm~\cite{schulman2017proximal}. }

\subsection{Regret bounds for Approximate-IDS}
We first introduce an auxiliary reward function $r'_h$ for the convenience of regret analysis. It  serves as a bridge connecting the approximated $\bar{r}_h$ to the real mutual information term in Algorithm \ref{alg}.
\begin{proposition}
\label{prop}
\textcolor{blue}{
For any $\epsilon$-value partition $\{\Theta_k^{\epsilon}\}_{k=1}^{K(\epsilon)}$ (Eqn.~\eqref{cover}) and the  surrogate environment $\tilde{\mathcal{E}}_t^{*}$ constructed by Eqn.~\eqref{lemma_tmp1},}
Define 
    \begin{equation}
          r'_h(s, a) \triangleq r_h(s, a) + \frac{\lambda}{2} \mathbb{E}_t \big[\mathrm{KL}_{s,a}^h(\tilde{\mathcal{E}}_t^{*}, \bar{\mathcal{E}}_t ) \big].
  \end{equation}
Then, for any policy $\pi$, we have
\begin{equation}
\label{lamep2}
    \bigg| \mathbb{E}_{\pi}^{\bar{\mathcal{E}}_t}\bigg[ \sum_{h=1}^{H}r'_h(s_h,a_h) \bigg]
    -\mathbb{E}_{\pi}^{\bar{\mathcal{E}}_t}\bigg[ \sum_{h=1}^{H}\bar{r}_h(s_h,a_h) \bigg]\bigg|  \leq \frac{\lambda}{2} \epsilon (1+\textcolor{blue}{\log\frac{B}{\beta}}).
\end{equation}
\end{proposition}

The proof is deferred to Appendix~\ref{propproof}. \textcolor{blue}{The proof indicates that the proposition may fail to hold when the surrogate environment is constructed using the $\ell_1$-distance or KL divergence.}
 To better understand this proposition, consider an extreme scenario: we divide the environment into the smallest units, with each $\Theta_k^{\epsilon}$ containing only one environment. We have $\mathcal{E}=\tilde{\mathcal{E}}_t^{*}$. The left hand of Eqn.~\eqref{lamep2} equals $0$, so Proposition~\ref{prop}  holds true. 
Since $ |a-b| \leq B|\log \frac{a}{b}| $ for any $a,b \in (0,B)$, we have $\ell_1(P,Q) \leq B \ell_g(P,Q)$. If we ignore the constant 
$B$, by fixing the $\epsilon$-value, our distance achieves a finer environmental partition. On this finer partition, $\mathcal{E}$ and $\tilde{\mathcal{E}}_t^{*}$ behave more similarly, allowing us to ensure that Proposition~\ref{prop} holds. 
Then, using Proposition \ref{prop}, we give the Bayesian regret bound for the Approximate-IDS algorithm. 
\begin{theorem}\label{thm}
 Given a Bayesian RLHF problem, for any $\epsilon>0$ and sufficiently large $T$, by choosing 
 $\lambda=\sqrt{\alpha^2TH/2\log(K(\epsilon))}$, we have the following regret upper bound for Algorithm \ref{alg2}:
    \begin{equation}
      \begin{aligned}
        BR_T(\app)
        \leq \alpha\sqrt{2TH\log(K(\epsilon))}
         +\left(1+\frac{(1+ \textcolor{blue}{ \log(B/\beta)})}{2}\sqrt{\frac{\alpha^2TH}{2\log(K(\epsilon))}}\right)T\epsilon+T_0.
      \end{aligned}
    \end{equation}
    By choosing a small $\epsilon$, the regret upper bound is of order $O\big(\sqrt{H^{3}T\log(K(\epsilon))}\big)$, matching that of Algorithm~\ref{alg} presented in Sec. \ref{regret_analysis}. 
\end{theorem}
\begin{proof}
First notice that 
\begin{equation}
\label{q}
\mathbb{E}_{\app^t}^{\bar{\mathcal{E}}_t}\bigg[ \sum_{h=1}^{H}\bar{r}_h(s_h,a_h) \bigg]
\overset{(a)}{\geq}
\mathbb{E}_{\ids^t}^{\bar{\mathcal{E}}_t}\bigg[ \sum_{h=1}^{H}\bar{r}_h(s_h,a_h) \bigg]\nonumber
\overset{(b)}{\geq }
\mathbb{E}_{\ids^t}^{\bar{\mathcal{E}}_t}\bigg[ \sum_{h=1}^{H}r_h(s_h,a_h) \bigg]\nonumber
\overset{(c)}{=}
\de_t\left[V_{1,\ids^t}^{\mathcal{E}}(s_1^t)\right],
\end{equation}
where (a) uses the optimality of $\app^t$, (b) follows from Eqn.~\eqref{rhbar},(c) uses  the definition of $\bar{\E}_t$ and the linearity of expectation. 

Therefore,
\begin{align}\label{w}
    \de_t\left[V_{1,\ids^t}^{\mathcal{E}}(s_1^t)\right]-\de_t\left[V_{1,\app^t}^{\mathcal{E}}(s_1^t)\right]&=\de_t\left[V_{1,\ids^t}^{\mathcal{E}}(s_1^t)\right]-\mathbb{E}_{\app^t}^{\bar{\mathcal{E}}_t}\bigg[ \sum_{h=1}^{H}r_h(s_h,a_h) \bigg]
    \nonumber\\&\leq \mathbb{E}_{\app^t}^{\bar{\mathcal{E}}_t}\bigg[ \sum_{h=1}^{H}\bar{r}_h(s_h,a_h)-r_h(s_h,a_h) \bigg]\nonumber\\
    &\leq \frac{\lambda\epsilon(1+ \textcolor{blue}{\log(B/\beta)})}{2}+\mathbb{E}_{\app^t}^{\bar{\mathcal{E}}_t}\bigg[ \sum_{h=1}^{H}r'_h(s_h,a_h)-r_h(s_h,a_h) \bigg]\nonumber\\
    &\leq \frac{\lambda\epsilon(1+ \textcolor{blue}{\log(B/\beta)})}{2}+\frac{\lambda}{2}\cdot\mathbb{I}_{t}^{\app^t}\left(\Tilde{\mathcal{E}}_t^*;(\mathcal{H}_t,\mathcal{R}_{t,H})\right),
\end{align}
where the first inequality is due to Eqn.~\eqref{q}, the second inequality is due to Proposition \ref{prop}, and the last inequality is due to Lemma \ref{mutual}.
Taking expectation in Eqn.~\eqref{w} with respect to $\mathcal{D}_t$ and then summing over $t\in[T]$, we obtain
\begin{equation}\label{r}
    BR_T(\app)-BR_T(\ids)\leq \frac{\lambda\epsilon(1+ \textcolor{blue}{\log(B/\beta)})T}{2}+\frac{\lambda}{2}\log(K(\epsilon)),
\end{equation}
where we use the same trick in Eqn.~\eqref{bound} to derive $\log(K(\epsilon))$ as an upper bound for $\sum_{t=1}^T \mathbb{I}_{t}^{\app^t}\left(\Tilde{\mathcal{E}}_t^*;(\mathcal{H}_t,\mathcal{R}_{t,H})\right)$. Finally, plugging the upper bound for $BR_T(\ids)$ (Eqn.~\eqref{e}) into  Eqn.~\eqref{r} and taking $\lambda=\sqrt{\alpha^2TH/2\log(K(\epsilon))}$
yields the proof of Theorem \ref{thm}.
\end{proof}

\subsection{\textcolor{blue}{Regret Analysis With $\ell_1$-distance}}
\textcolor{blue}{ 
Using the standard $\ell_1$-distance for both environment partitioning and surrogate environment construction, rather than the proposed $\ell_g$-distance, would necessitate significantly stronger and less realistic assumptions, and result in worse theoretical guarantees.
To ensure the validity of Proposition~\ref{prop}, the following stronger assumptions should be additionally imposed.
\begin{assumption}
     For any $P \in \Theta$,  there exist $\beta, B>0$ such that
	 \begin{equation}
		 \label{tmp4}
		\beta \leq \inf_{o,x} \{ P(x|o)\} \leq \sup_{o,x} \{ P(x|o)\} \leq B. 
	 \end{equation}
\end{assumption}
This assumption excludes zero probability densities for transitions/rewards, limiting its practical applicability significantly. Under this assumption, Proposition~\ref{prop} can be reformulated as follows.
\begin{proposition}
\label{prop_1}
For any $\epsilon$-value partition $\{\Theta_k^{\epsilon}\}_{k=1}^{K(\epsilon)}$ (partitioned by $\ell_1$-distance) and the  surrogate environment $\tilde{\mathcal{E}}_t^{*}$ (constructed by $\ell_1$-distance ),
 we have
\begin{equation}
    \bigg| \mathbb{E}_{\pi}^{\bar{\mathcal{E}}_t}\bigg[ \sum_{h=1}^{H}r'_h(s_h,a_h) \bigg]
    -\mathbb{E}_{\pi}^{\bar{\mathcal{E}}_t}\bigg[ \sum_{h=1}^{H}\bar{r}_h(s_h,a_h) \bigg]\bigg|  \leq \frac{\lambda}{2} \epsilon (1+\frac{B}{\beta}).
\end{equation}
\end{proposition}
The proof of Proposition~\ref{prop_1} can be found in Appendix~\ref{propproof_1}. Using Proposition~\ref{prop_1},  the Bayesian regret bound for Algorithm~\ref{alg2} under $\ell_1$-distance is of order 
\[ \tilde{O}(\frac{1}{\beta}\sqrt{H^3T}) .\]
In comparison, the regret bound derived using  $\ell_g$  is of order 
\[ \tilde{O}((\log\frac{1}{\beta})\sqrt{H^3T}) .\]
Although we have also derived a regret bound of sublinear dependence on $H$ and $T$ using the $\ell_1$-distance, this bound includes a factor of $\frac{1}{\beta}$. When $\beta$ is small, the bound becomes excessively large. Moreover, the underlying assumption  of this bound contradicts real-world scenarios, as common MDPs permit zero transition and reward probabilities in practical settings—such as the boundaries in Go or obstacles in robotic navigation. In light of these limitations, our proposed $\ell_g$-based approach demonstrates significant superiority. }

\section{Conclusion}

In this paper, we introduced novel information-directed sampling (IDS) algorithms to address key challenges in the RLHF problem, a critical component of LLM training. Our method improves the sample efficiency by maximizing both the value function and the mutual information between the  (surrogate) environment and trajectories. We also developed a computationally efficient Approximate-IDS algorithm suitable for real-world applications while maintaining the regret bound order of the original method. A potentially practical implication of our sample-efficient algorithms is their ability to align LLMs to human values with less human feedback while maintaining similar performance, thereby reducing the cost and time of LLM training. Additionally, our findings highlight the value of information theory in the rapidly evolving era of LLMs.


\bibliographystyle{ims}
\bibliography{main}

\begin{thebibliography}{41}
\expandafter\ifx\csname natexlab\endcsname\relax\def\natexlab#1{#1}\fi
\expandafter\ifx\csname url\endcsname\relax
  \def\url#1{\texttt{#1}}\fi
\expandafter\ifx\csname urlprefix\endcsname\relax\def\urlprefix{}\fi

\bibitem[{Achiam et~al.(2023)Achiam, Adler, Agarwal, Ahmad, Akkaya, Aleman, Almeida, Altenschmidt, Altman, Anadkat et~al.}]{achiam2023gpt}
\text{Achiam, J.}, \text{Adler, S.}, \text{Agarwal, S.}, \text{Ahmad, L.}, \text{Akkaya, I.}, \text{Aleman, F.~L.}, \text{Almeida, D.}, \text{Altenschmidt, J.}, \text{Altman, S.}, \text{Anadkat, S.} \text{et~al.} (2023).
\newblock Gpt-4 technical report.
\newblock \textit{arXiv preprint arXiv:2303.08774}.

\bibitem[{Bai et~al.(2025)Bai, Zhang, Qiu, Zhang, Xu and Li}]{bai2025online}
\text{Bai, C.}, \text{Zhang, Y.}, \text{Qiu, S.}, \text{Zhang, Q.}, \text{Xu, K.} and \text{Li, X.} (2025).
\newblock Online preference alignment for language models via count-based exploration.
\newblock In \textit{International Conference on Learning Representations (ICLR)}.

\bibitem[{Bradley and Terry(1952)}]{bradley1952rank}
\text{Bradley, R.~A.} and \text{Terry, M.~E.} (1952).
\newblock Rank analysis of incomplete block designs: I. the method of paired comparisons.
\newblock \textit{Biometrika}, \textbf{39} 324--345.

\bibitem[{Bubeck and Sellke(2020)}]{bubeck2020first}
\text{Bubeck, S.} and \text{Sellke, M.} (2020).
\newblock First-order bayesian regret analysis of thompson sampling.
\newblock In \textit{Algorithmic Learning Theory}. PMLR.

\bibitem[{Busa-Fekete et~al.(2014)Busa-Fekete, Sz{\"o}r{\'e}nyi, Weng, Cheng and H{\"u}llermeier}]{busa2014preference}
\text{Busa-Fekete, R.}, \text{Sz{\"o}r{\'e}nyi, B.}, \text{Weng, P.}, \text{Cheng, W.} and \text{H{\"u}llermeier, E.} (2014).
\newblock Preference-based reinforcement learning: evolutionary direct policy search using a preference-based racing algorithm.
\newblock \textit{Machine learning}, \textbf{97} 327--351.

\bibitem[{Chen et~al.(2022)Chen, Zhong, Yang, Wang and Wang}]{chen2022human}
\text{Chen, X.}, \text{Zhong, H.}, \text{Yang, Z.}, \text{Wang, Z.} and \text{Wang, L.} (2022).
\newblock Human-in-the-loop: Provably efficient preference-based reinforcement learning with general function approximation.
\newblock In \textit{International Conference on Machine Learning}. PMLR.

\bibitem[{Dong and Van~Roy(2018)}]{dong2018information}
\text{Dong, S.} and \text{Van~Roy, B.} (2018).
\newblock An information-theoretic analysis for thompson sampling with many actions.
\newblock \textit{Advances in Neural Information Processing Systems}, \textbf{31}.

\bibitem[{Foster et~al.(2021)Foster, Kakade, Qian and Rakhlin}]{foster2021statistical}
\text{Foster, D.~J.}, \text{Kakade, S.~M.}, \text{Qian, J.} and \text{Rakhlin, A.} (2021).
\newblock The statistical complexity of interactive decision making.
\newblock \textit{arXiv preprint arXiv:2112.13487}.

\bibitem[{F{\"u}rnkranz et~al.(2012)F{\"u}rnkranz, H{\"u}llermeier, Cheng and Park}]{furnkranz2012preference}
\text{F{\"u}rnkranz, J.}, \text{H{\"u}llermeier, E.}, \text{Cheng, W.} and \text{Park, S.-H.} (2012).
\newblock Preference-based reinforcement learning: a formal framework and a policy iteration algorithm.
\newblock \textit{Machine learning}, \textbf{89} 123--156.

\bibitem[{Ghosal and van~der Vaart(2017)}]{ghosal2017fundamentals}
\text{Ghosal, S.} and \text{van~der Vaart, A.~W.} (2017).
\newblock \textit{Fundamentals of nonparametric Bayesian inference}, vol.~44.
\newblock Cambridge University Press.

\bibitem[{Haninger et~al.(2022)Haninger, Hegeler and Peternel}]{haninger2022model}
\text{Haninger, K.}, \text{Hegeler, C.} and \text{Peternel, L.} (2022).
\newblock Model predictive control with gaussian processes for flexible multi-modal physical human robot interaction.
\newblock In \textit{2022 international conference on robotics and automation (ICRA)}. IEEE.

\bibitem[{Haninger et~al.(2023)Haninger, Hegeler and Peternel}]{haninger2023model}
\text{Haninger, K.}, \text{Hegeler, C.} and \text{Peternel, L.} (2023).
\newblock Model predictive impedance control with gaussian processes for human and environment interaction.
\newblock \textit{Robotics and Autonomous Systems}, \textbf{165} 104431.

\bibitem[{Hao and Lattimore(2022)}]{hao2022regret}
\text{Hao, B.} and \text{Lattimore, T.} (2022).
\newblock Regret bounds for information-directed reinforcement learning.
\newblock \textit{Advances in neural information processing systems}, \textbf{35} 28575--28587.

\bibitem[{Hao et~al.(2021)Hao, Lattimore and Deng}]{hao2021information}
\text{Hao, B.}, \text{Lattimore, T.} and \text{Deng, W.} (2021).
\newblock Information directed sampling for sparse linear bandits.
\newblock \textit{Advances in Neural Information Processing Systems}, \textbf{34} 16738--16750.

\bibitem[{Hao et~al.(2022)Hao, Lattimore and Qin}]{hao2022contextual}
\text{Hao, B.}, \text{Lattimore, T.} and \text{Qin, C.} (2022).
\newblock Contextual information-directed sampling.
\newblock In \textit{International Conference on Machine Learning}. PMLR.

\bibitem[{Ji et~al.(2024)Ji, He and Gu}]{ji2024reinforcement}
\text{Ji, K.}, \text{He, J.} and \text{Gu, Q.} (2024).
\newblock Reinforcement learning from human feedback with active queries.
\newblock \textit{arXiv preprint arXiv:2402.09401}.

\bibitem[{Kirschner et~al.(2021)Kirschner, Lattimore, Vernade and Szepesv{\'a}ri}]{kirschner2021asymptotically}
\text{Kirschner, J.}, \text{Lattimore, T.}, \text{Vernade, C.} and \text{Szepesv{\'a}ri, C.} (2021).
\newblock Asymptotically optimal information-directed sampling.
\newblock In \textit{Conference on Learning Theory}. PMLR.

\bibitem[{Li et~al.(2024)Li, Zhao and Gu}]{li2024feel}
\text{Li, X.}, \text{Zhao, H.} and \text{Gu, Q.} (2024).
\newblock Feel-good thompson sampling for contextual dueling bandits.
\newblock \textit{arXiv preprint arXiv:2404.06013}.

\bibitem[{Liu et~al.(2018)Liu, Buccapatnam and Shroff}]{liu2018information}
\text{Liu, F.}, \text{Buccapatnam, S.} and \text{Shroff, N.} (2018).
\newblock Information directed sampling for stochastic bandits with graph feedback.
\newblock In \textit{Proceedings of the AAAI Conference on Artificial Intelligence}, vol.~32.

\bibitem[{Moradipari et~al.(2023)Moradipari, Pedramfar, Shokrian~Zini and Aggarwal}]{moradipari2023improved}
\text{Moradipari, A.}, \text{Pedramfar, M.}, \text{Shokrian~Zini, M.} and \text{Aggarwal, V.} (2023).
\newblock Improved bayesian regret bounds for thompson sampling in reinforcement learning.
\newblock \textit{Advances in Neural Information Processing Systems}, \textbf{36} 23557--23569.

\bibitem[{Osband et~al.(2013)Osband, Russo and Van~Roy}]{osband2013more}
\text{Osband, I.}, \text{Russo, D.} and \text{Van~Roy, B.} (2013).
\newblock (more) efficient reinforcement learning via posterior sampling.
\newblock \textit{Advances in Neural Information Processing Systems}, \textbf{26}.

\bibitem[{Ouyang et~al.(2022)Ouyang, Wu, Jiang, Almeida, Wainwright, Mishkin, Zhang, Agarwal, Slama, Ray et~al.}]{ouyang2022training}
\text{Ouyang, L.}, \text{Wu, J.}, \text{Jiang, X.}, \text{Almeida, D.}, \text{Wainwright, C.}, \text{Mishkin, P.}, \text{Zhang, C.}, \text{Agarwal, S.}, \text{Slama, K.}, \text{Ray, A.} \text{et~al.} (2022).
\newblock Training language models to follow instructions with human feedback.
\newblock \textit{Advances in neural information processing systems}, \textbf{35} 27730--27744.

\bibitem[{Pacchiano et~al.(2021)Pacchiano, Saha and Lee}]{pacchiano2021dueling}
\text{Pacchiano, A.}, \text{Saha, A.} and \text{Lee, J.} (2021).
\newblock Dueling rl: reinforcement learning with trajectory preferences.
\newblock \textit{arXiv preprint arXiv:2111.04850}.

\bibitem[{Rafailov et~al.(2024)Rafailov, Sharma, Mitchell, Manning, Ermon and Finn}]{rafailov2024direct}
\text{Rafailov, R.}, \text{Sharma, A.}, \text{Mitchell, E.}, \text{Manning, C.~D.}, \text{Ermon, S.} and \text{Finn, C.} (2024).
\newblock Direct preference optimization: Your language model is secretly a reward model.
\newblock \textit{Advances in Neural Information Processing Systems}, \textbf{36}.

\bibitem[{Russo and Van~Roy(2014)}]{russo2014learning}
\text{Russo, D.} and \text{Van~Roy, B.} (2014).
\newblock Learning to optimize via information-directed sampling.
\newblock \textit{Advances in neural information processing systems}, \textbf{27}.

\bibitem[{Russo and Van~Roy(2016)}]{russo2016information}
\text{Russo, D.} and \text{Van~Roy, B.} (2016).
\newblock An information-theoretic analysis of thompson sampling.
\newblock \textit{Journal of Machine Learning Research}, \textbf{17} 1--30.

\bibitem[{Saha(2021)}]{saha2021optimal}
\text{Saha, A.} (2021).
\newblock Optimal algorithms for stochastic contextual preference bandits.
\newblock \textit{Advances in Neural Information Processing Systems}, \textbf{34} 30050--30062.

\bibitem[{Saha et~al.(2023)Saha, Pacchiano and Lee}]{saha2023dueling}
\text{Saha, A.}, \text{Pacchiano, A.} and \text{Lee, J.} (2023).
\newblock Dueling rl: Reinforcement learning with trajectory preferences.
\newblock In \textit{International Conference on Artificial Intelligence and Statistics}. PMLR.

\bibitem[{Schulman et~al.(2017)Schulman, Wolski, Dhariwal, Radford and Klimov}]{schulman2017proximal}
\text{Schulman, J.}, \text{Wolski, F.}, \text{Dhariwal, P.}, \text{Radford, A.} and \text{Klimov, O.} (2017).
\newblock Proximal policy optimization algorithms.
\newblock \textit{arXiv preprint arXiv:1707.06347}.

\bibitem[{Sekhari et~al.(2024)Sekhari, Sridharan, Sun and Wu}]{sekhari2024contextual}
\text{Sekhari, A.}, \text{Sridharan, K.}, \text{Sun, W.} and \text{Wu, R.} (2024).
\newblock Contextual bandits and imitation learning with preference-based active queries.
\newblock \textit{Advances in Neural Information Processing Systems}, \textbf{36}.

\bibitem[{Taranovic et~al.(2022)Taranovic, Kupcsik, Freymuth and Neumann}]{taranovic2022adversarial}
\text{Taranovic, A.}, \text{Kupcsik, A.~G.}, \text{Freymuth, N.} and \text{Neumann, G.} (2022).
\newblock Adversarial imitation learning with preferences.
\newblock In \textit{The Eleventh International Conference on Learning Representations}.

\bibitem[{Tossou et~al.(2019)Tossou, Basu and Dimitrakakis}]{tossou2019nearoptimaloptimisticreinforcementlearning}
\text{Tossou, A.}, \text{Basu, D.} and \text{Dimitrakakis, C.} (2019).
\newblock Near-optimal optimistic reinforcement learning using empirical bernstein inequalities.
\newblock \textit{arXiv preprint arXiv:1905.12425}.

\bibitem[{Touvron et~al.(2023)Touvron, Martin, Stone, Albert, Almahairi, Babaei, Bashlykov, Batra, Bhargava, Bhosale et~al.}]{touvron2023llama}
\text{Touvron, H.}, \text{Martin, L.}, \text{Stone, K.}, \text{Albert, P.}, \text{Almahairi, A.}, \text{Babaei, Y.}, \text{Bashlykov, N.}, \text{Batra, S.}, \text{Bhargava, P.}, \text{Bhosale, S.} \text{et~al.} (2023).
\newblock Llama 2: Open foundation and fine-tuned chat models.
\newblock \textit{arXiv preprint arXiv:2307.09288}.

\bibitem[{Wu and Sun(2023)}]{wu2023making}
\text{Wu, R.} and \text{Sun, W.} (2023).
\newblock Making rl with preference-based feedback efficient via randomization.
\newblock \textit{arXiv preprint arXiv:2310.14554}.

\bibitem[{Wu et~al.(2021)Wu, Wang, Evans, Tenenbaum, Parkes and Kleiman-Weiner}]{wu2021too}
\text{Wu, S.~A.}, \text{Wang, R.~E.}, \text{Evans, J.~A.}, \text{Tenenbaum, J.~B.}, \text{Parkes, D.~C.} and \text{Kleiman-Weiner, M.} (2021).
\newblock Too many cooks: Bayesian inference for coordinating multi-agent collaboration.
\newblock \textit{Topics in Cognitive Science}, \textbf{13} 414--432.

\bibitem[{Xie et~al.(2024)Xie, Foster, Krishnamurthy, Rosset, Awadallah and Rakhlin}]{xie2024exploratorypreferenceoptimizationharnessing}
\text{Xie, T.}, \text{Foster, D.~J.}, \text{Krishnamurthy, A.}, \text{Rosset, C.}, \text{Awadallah, A.} and \text{Rakhlin, A.} (2024).
\newblock Exploratory preference optimization: Harnessing implicit q*-approximation for sample-efficient rlhf.
\newblock \textit{arXiv preprint arXiv:2405.21046}.

\bibitem[{Xu et~al.(2020)Xu, Wang, Yang, Singh and Dubrawski}]{xu2020preference}
\text{Xu, Y.}, \text{Wang, R.}, \text{Yang, L.}, \text{Singh, A.} and \text{Dubrawski, A.} (2020).
\newblock Preference-based reinforcement learning with finite-time guarantees.
\newblock \textit{Advances in Neural Information Processing Systems}, \textbf{33} 18784--18794.

\bibitem[{Ye et~al.(2024)Ye, Xiong, Zhang, Jiang and Zhang}]{ye2024theoretical}
\text{Ye, C.}, \text{Xiong, W.}, \text{Zhang, Y.}, \text{Jiang, N.} and \text{Zhang, T.} (2024).
\newblock A theoretical analysis of nash learning from human feedback under general kl-regularized preference.
\newblock \textit{arXiv preprint arXiv:2402.07314}.

\bibitem[{Zhang et~al.(2024)Zhang, Bai, Hu, Wang and Li}]{zhang2024provably}
\text{Zhang, Q.}, \text{Bai, C.}, \text{Hu, S.}, \text{Wang, Z.} and \text{Li, X.} (2024).
\newblock Provably efficient information-directed sampling algorithms for multi-agent reinforcement learning.
\newblock \textit{arXiv preprint arXiv:2404.19292}.

\bibitem[{Zhu et~al.(2023)Zhu, Jordan and Jiao}]{zhu2023principled}
\text{Zhu, B.}, \text{Jordan, M.} and \text{Jiao, J.} (2023).
\newblock Principled reinforcement learning with human feedback from pairwise or k-wise comparisons.
\newblock In \textit{International Conference on Machine Learning}. PMLR.

\bibitem[{Ziegler et~al.(2019)Ziegler, Stiennon, Wu, Brown, Radford, Amodei, Christiano and Irving}]{ziegler2019fine}
\text{Ziegler, D.~M.}, \text{Stiennon, N.}, \text{Wu, J.}, \text{Brown, T.~B.}, \text{Radford, A.}, \text{Amodei, D.}, \text{Christiano, P.} and \text{Irving, G.} (2019).
\newblock Fine-tuning language models from human preferences.
\newblock \textit{arXiv preprint arXiv:1909.08593}.

\end{thebibliography}

\newpage
\appendix

\section{Proofs of Theorem and Proposition}
\label{appendix_A}
\subsection{Proof of Theorem \ref{theorem}}\label{proof_theorem1} 
\begin{unnumberedtheorem}
    Given a Bayesian RLHF problem, for any $\epsilon>0$ and sufficiently large $T$, by choosing $\lambda=\sqrt{\alpha^2TH/\log(K(\epsilon))}$, we have
    \begin{equation}
        BR_T(\ids)\leq \alpha \sqrt{TH\log(K(\epsilon))}+T\epsilon+T_0,
    \end{equation}
    where $T_0$ is a fixed positive integer that is independent of $T$.  
\end{unnumberedtheorem}
\begin{proof}
We divide the proof into 5 steps. First, we point out that by the law of total expectation, we can rewrite the Bayesian regret as 
\begin{equation}\label{3_2}
    BR_T(\pi_{\text{IDS}})=\sum_{t=1}^{T} \mathbb{E}_{\mathcal{D}_t}\left[ \mathbb{E}_{\mathcal{E} \sim \mathbb{P}(\cdot|\mathcal{D}_t)}\left[V_{1,\pi_{\mathcal{E}}^{*}}^{\mathcal{E}}(s_1^t)-V_{1,\pi^{t}}^{\mathcal{E}}(s_1^t)\right]\right],
\end{equation}
whose form is more convenient for analysis.


\textbf{Step 1.} Reduce $BR_T(\ids)$ to the surrogate environment, and convert $BR_T(\ids)$ into $BR_T(\ts)$. 
By Lemma \ref{partition} and the optimality of $\ids$, we have 
\begin{equation}
    \begin{aligned}
     \label{step1}
    BR_T(\ids)&=\sum_{t=1}^{T} \de_{\mathcal{D}_t}\left[ \mathbb{E}_{\mathcal{E} \sim \mathbb{P}(\cdot|\mathcal{D}_t)}\left[V_{1,\pi_{\mathcal{E}}^{*}}^{\mathcal{E}}(s_1^t)-V_{1,\ids^{t}}^{\mathcal{E}}(s_1^t)\right]\right]\\
&=\sum_{t=1}^{T}\de_{\mathcal{D}_t}\left[\mathbb{E}_{t}\left[V_{1,\pi_{\mathcal{E}}^{*}}^{\mathcal{E}}(s_1^t)-V_{1,\ids^{t}}^{\mathcal{E}}(s_1^t)\right]-\epsilon-\frac{\lambda}{2}\mathbb{I}_{t}^{\ids^t}\left(\tilde{\mathcal{E}}_{t}^{*};(\mathcal{H}_t,\mathcal{R}_{t,H})\right)\right]\\
&\qquad +\frac{\lambda}{2}\sum_{t=1}^{T}\de_{\mathcal{D}_t}\left[\mathbb{I}_{t}^{\ids^t}\left(\tilde{\mathcal{E}}_{t}^{*};(\mathcal{H}_t,\mathcal{R}_{t,H})\right)\right]+T\epsilon\\
    &\overset{(a)}{\leq} \sum_{t=1}^{T}\de_{\mathcal{D}_t}\left[\mathbb{E}_{t}\left[V_{1,\pi_{\mathcal{E}}^{*}}^{\mathcal{E}}(s_1^t)-V_{1,\ts^{t}}^{\mathcal{E}}(s_1^t)\right]-\epsilon-\frac{\lambda}{2}\mathbb{I}_{t}^{\ts^t}\left(\tilde{\mathcal{E}}_{t}^{*};(\mathcal{H}_t,\mathcal{R}_{t,H})\right)\right]\\
    &\qquad +\frac{\lambda}{2}\sum_{t=1}^{T}\de_{\mathcal{D}_t}\left[\mathbb{I}_{t}^{\ids^t}\left(\tilde{\mathcal{E}}_{t}^{*};(\mathcal{H}_t,\mathcal{R}_{t,H})\right)\right]+T\epsilon\\
    &\overset{(b)}{\leq}\sum_{t=1}^{T}\de_{\mathcal{D}_t}\left[\mathbb{E}_{t}\left[V_{1,\pi_{\mathcal{E}}^{*}}^{\tilde{\E}_t^*}(s_1^t)-V_{1,\ts^{t}}^{\tilde{\E}_t^*}(s_1^t)\right]-\frac{\lambda}{2}\mathbb{I}_{t}^{\ts^t}\left(\tilde{\mathcal{E}}_{t}^{*};(\mathcal{H}_t,\mathcal{R}_{t,H})\right)\right]\\
    &\qquad +\frac{\lambda}{2}\sum_{t=1}^{T}\de_{\mathcal{D}_t}\left[\mathbb{I}_{t}^{\ids^t}\left(\tilde{\mathcal{E}}_{t}^{*};(\mathcal{H}_t,\mathcal{R}_{t,H})\right)\right]+T\epsilon,  
    \end{aligned}
\end{equation}
where (a) uses the optimality of $\ids^t$, (b) uses  Lemma \ref{partition}.

For the first term in Eqn.~\eqref{step1}, using the basic fact that $A-\lambda B/2\leq A^2/2\lambda B$ for $B,\lambda\geq0$, we have 
\begin{equation}
    \de_{t}\left[V_{1,\pi_{\E}^{*}}^{\tilde{\E}_t^*}(s_1^t)-V_{1,\ts^{t}}^{\tilde{\E}_t^*}(s_1^t)\right]-\frac{\lambda}{2}\mathbb{I}_{t}^{\ts^t}\left(\tilde{\E}_{t}^{*};(\mathcal{H}_t,\mathcal{R}_{t,H})\right)\leq\frac{1}{2\lambda}\frac{\left(\de_t\left[ V_{1,\pi_{\E}^{*}}^{\tilde{\E}_{t}^{*}}(s_1^t)-V_{1,\ts^{t}}^{\tilde{\E}_{t}^{*}}(s_1^t) \right]\right)^2}{ \mathbb{I}_{t}^{\ts^t}\left(\tilde{\mathcal{E}}_{t}^{*};(\mathcal{H}_t,\mathcal{R}_{t,H})\right)}\triangleq \frac{1}{2\lambda}\Gamma_t^{\ts^t}.
\end{equation}
where we introduce the tool of \emph{information ratio} $\Gamma_t^{\ts^t}$ for ease of analysis.

Let $\zeta$ be a discrete random variable taking values in $\{1,...,K(\epsilon)\}$ such that $\zeta=k$ if and only if  $\mathcal{E} \in \Theta_k^{\epsilon}$. 
From the construction of the surrogate environment (Eqn.~\eqref{lemma_tmp1}),  the distribution of $\tilde{\mathcal{E}}_t^{*}$ depend on $\E$ only through $\zeta$, i.e., $\tilde{\mathcal{E}}_t^{*}$ and $\E$ are independent conditioning on $\zeta$. 

For the second term in Eqn.~\eqref{step1}, we have
\begin{align}
\label{bound}
    \sum_{t=1}^{T}\de_{\mathcal{D}_t}\left[\mathbb{I}_{t}^{\ids^t}\left(\tilde{\mathcal{E}}_{t}^{*};(\mathcal{H}_t,\mathcal{R}_{t,H})\right)\right]
    \leq\sum_{t=1}^{T}\de_{\mathcal{D}_t}\left[\mathbb{I}_{t}^{\ids^t}\left(\zeta;(\mathcal{H}_t,\mathcal{R}_{t,H})\right)\right]
    =\mathbb{I}(\zeta;\md_{T+1})
    \leq \mathbb{H}(\zeta)
    \leq \log(K(\epsilon)),
\end{align}
where the first inequality is due to data processing inequality, the second equality is due to the chain rule of mutual information, and the last two inequalities follow from the basic definition of entropy. Therefore, we derive an upper bound for $BR_T(\ids)$ as follows 
\begin{equation}
    BR_T(\ids)\leq \frac{1}{2\lambda}\de\left[\sum_{t=1}^{T}\Gamma_t^{\ts^t}\right]+\frac{\lambda}{2}\log(K(\epsilon))+T\epsilon.
\end{equation}

\textbf{Step 2} (Bound $\Gamma_t^{\ts^t}$). Before stepping into technical details, we need to introduce several concepts. First, 
the state-action occupancy function $d_{h,\pi}^{\mathcal{E}}:\mathcal{S}\times\mathcal{A}\to\mathbb{R}$ at step $h$ under policy $\pi$ and environment $\E$, is defined as the Radon-Nikodym derivative of the state-action occupancy measure $\mathbb{P}_{\pi}^{\mathcal{E}}((s_h,a_h)=\cdot)$ with regard to the base probability measure $\mu_{\mathcal{S}\times\mathcal{A}}$ on the product space $\mathcal{S}\times\mathcal{A}$, i.e.,
\begin{equation*}
    d_{h,\pi}^{\mathcal{E}}(s,a)\triangleq\frac{\mathrm{d}\mathbb{P}_{\pi}^{\mathcal{E}}(s_h=s,a_h=a)}{\mathrm{d}\mu_{\mathcal{S}\times\mathcal{A}}}.
\end{equation*}
For convenience of analysis, we assume that $d_{h,\pi}^{\mathcal{E}}(s,a)$ is measurable and upper bounded for all $\pi,\E,s,a,h$. Recall that, the mean environment $\bar{\mathcal{E}}_t$ is defined to satisfy $P_h^{\bar{\mathcal{E}}_t}(\cdot|s,a)=\mathbb{E}_t[P_h^{\mathcal{E}}(\cdot |s,a)]$ and $R_h^{\bar{\mathcal{E}}_t}(\cdot|s,a)=\mathbb{E}_t[R_h^{\mathcal{E}}(\cdot |s,a)]$ for all $s\in\mathcal{S}$ and $a\in\mathcal{A}$. By the definition of $d_{h,\pi}^\E$, the following equality also holds: $d_{h,\pi}^{\bar{\mathcal{E}}_t}(s,a)=\mathbb{E}_t[d_{h,\pi}^{\mathcal{E}}(s,a)]$. 
One important property of the mean environment is that the posterior mean of the surrogate environment $\de_t[\tilde{\E}_t^*]$ coincides with that of the whole environment $\bar{\E}_t$. To check this, using the property of conditional expectation:
\begin{align}
    \de_t[\tilde{\mathcal{E}}_t^{*}]&=\sum_{k=1}^K\dP(\E\in\Theta_k^\epsilon)\cdot\de_t[\tilde{\mathcal{E}}_t^{*}|\E\in\Theta_k^\epsilon]\nonumber\\
    &\overset{(a)}{=}\sum_{k=1}^K\dP(\E\in\Theta_k^\epsilon)\cdot\de_t[\tilde{\mathcal{E}}_{k,t}^{*}]\nonumber\\
    &\overset{(b)}{=}\sum_{k=1}^K\dP(\E\in\Theta_k^\epsilon)\cdot\tilde{\mathcal{E}}_{k,t}^{*}\nonumber\\
    &\overset{(c)}{=}\sum_{k=1}^K\dP(\E\in\Theta_k^\epsilon)\cdot\de_t[\E|\E\in\Theta_k^\epsilon]=\bar{\E}_t,\label{b7}
\end{align}
where (a)  comes from the definition of $\tilde{\mathcal{E}}_t^{*}$ and $\tilde{\mathcal{E}}_{k,t}^{*} $, (b) and (c) uses the following  fact  \[ \de_t[ \tilde{\mathcal{E}}_{k,t}^{*}]= \de_t[\de_t[\E|\E\in\Theta_k^\epsilon] ] = \de_t[\E|\E\in\Theta_k^\epsilon] = \tilde{\mathcal{E}}_t^{*} \]

Finally, we denote the value function difference as
\begin{equation}
    \Delta_h^{\tilde{\mathcal{E}}_t^{*}}(s,a)\triangleq\mathbb{E}_{(s',r') \sim (P_h^{\tilde{\mathcal{E}}_t^{*}} \otimes R_h^{\tilde{\mathcal{E}}_t^{*}})(\cdot|s,a)}\left[ r'+V_{h+1,\pi_{\mathcal{E}}^{*}}^{\tilde{\mathcal{E}}_t^{*}}(s')\right]-\mathbb{E}_{(s',r') \sim (P_h^{\bar{\mathcal{E}}_t}\otimes R_h^{\bar{\mathcal{E}}_t})(\cdot|s,a)}\left[ r'+V_{h+1,\pi_{\mathcal{E}}^{*}}^{\tilde{\mathcal{E}}_t^{*}}(s') \right]. 
\end{equation}

Now we are ready to give an upper bound for $\Gamma_t^{\ts^t}$. We hope to use Lemma \ref{lemma} to rewrite the numerator 
$$\left(\de_t\left[ V_{1,\pi_{\E}^{*}}^{\tilde{\E}_{t}^{*}}(s_1^t)-V_{1,\ts^{t}}^{\tilde{\E}_{t}^{*}}(s_1^t) \right]\right)^2.$$ 
However, Lemma \ref{lemma} can only be applied to handle the difference between two value functions with the same policy and different environments, while in $V_{1,\pi_{\E}^{*}}^{\tilde{\E}_{t}^{*}}(s_1^t)$ and $V_{1,\ts^{t}}^{\tilde{\E}_{t}^{*}}(s_1^t)$, the environments are the same and the policies are different. For the purpose of ``unifying'' the policy, we use Eqn.~\eqref{b7} and note that $\ts$ is independent of $\tilde{\mathcal{E}}_t^{*}$, yielding $$\de_t\left[V_{1,\ts^t}^{\tilde{\E}_{t}^{*}}(s_1^t)\right]=\de_t\left[V_{1,\ts^t}^{\bar{\E}_{t}^{*}}(s_1^t)\right].$$ Furthermore, conditioned on $\mathcal{D}_t$, $\pi_{\text{TS}}^t$ and $\pi_{\E}^{*}$ are identically distributed, and are both independent of $\bar{\E}_t$. This implies $$\de_t\left[V_{1,\ts^t}^{\bar{\E}_{t}^{*}}(s_1^t)\right]=\de_t\left[V_{1,\pi_\E^*}^{\bar{\E}_{t}^{*}}(s_1^t)\right].$$ Therefore, by Lemma \ref{lemma}, we have
\begin{align}
    \mathbb{E}_t\left[ V_{1,\pi_{\mathcal{E}}^{*}}^{\tilde{\mathcal{E}}_{t}^{*}}(s_1^t)-V_{1,\pi_{\text{TS}}^{t}}^{\tilde{\mathcal{E}}_{t}^{*}}(s_1^t) \right]&=\mathbb{E}_t\left[ V_{1,\pi_{\mathcal{E}}^{*}}^{\tilde{\mathcal{E}}_{t}^{*}}(s_1^t)-V_{1,\pi_{\mathcal{E}}^{*}}^{\bar{\mathcal{E}}_{t}}(s_1^t) \right]\nonumber\\
    &= \sum_{h=1}^{H}\mathbb{E}_t \mathbb{E}_{\pi_{\mathcal{E}}^{*}}^{\bar{\mathcal{E}}_t}\left[\Delta_h^{\tilde{\mathcal{E}}_t^{*}}(s,a)\right]\nonumber\\
    &=  \sum_{h=1}^{H}\mathbb{E}_t \left[ \int_{\mathcal{S}\times \mathcal{A}}d_{h,\pi_{\mathcal{E}}^{*}}^{\bar{\mathcal{E}}_t}(s,a)\Delta_h^{\tilde{\mathcal{E}}_t^{*}}(s,a) \mathrm{d}\mu_{\mathcal{S}\times \mathcal{A}} \right],
\end{align}
where the notation of $d_{h,\pi_{\mathcal{E}}^{*}}^{\bar{\mathcal{E}}_t}(s,a)$ and $\Delta_h^{\tilde{\mathcal{E}}_t^{*}}(s,a)$ are introduced to simplify the formula. Following~\cite{moradipari2023improved}, we
define
\begin{equation}
    \mathcal{I}^t\triangleq \sum_{h=1}^H\de_t\mathbb{E}_{\ts^t}^{\bar{\mathcal{E}}_t}\left[\frac{\Delta_h^{\tilde{\E}_t^{*}}(s,a)^2}{\alpha_\E^2}\right], \quad\mathcal{T}^t\triangleq\sum_{h=1}^H\int_{\de_t[d_{h,\pi_{\mathcal{E}}^{*}}^{\bar{\mathcal{E}}_t}(s,a)]\neq 0}\frac{\de_t\left[\alpha_\E^2\cdot d_{h,\pi_{\mathcal{E}}^{*}}^{\bar{\mathcal{E}}_t}(s,a)^2\right]}{\de_t\left[d_{h,\pi_{\mathcal{E}}^{*}}^{\bar{\mathcal{E}}_t}(s,a)\right]}\mathrm{d}\mu_{\mathcal{S}\times \mathcal{A}}.
\end{equation}
By the Cauchy-Schwarz inequality, we have 
\begin{align}
    &\sum_{h=1}^{H}\mathbb{E}_t \left[ \int_{\mathcal{S}\times \mathcal{A}}d_{h,\pi_{\mathcal{E}}^{*}}^{\bar{\mathcal{E}}_t}(s,a)\Delta_h^{\tilde{\mathcal{E}}_t^{*}}(s,a) \mathrm{d}\mu_{\mathcal{S}\times \mathcal{A}} \right]\nonumber\\
    &=\sum_{h=1}^{H}\mathbb{E}_t \left[ \int_{\de_t[d_{h,\pi_{\mathcal{E}}^{*}}^{\bar{\mathcal{E}}_t}(s,a)]\neq 0}d_{h,\pi_{\mathcal{E}}^{*}}^{\bar{\mathcal{E}}_t}(s,a)\Delta_h^{\tilde{\mathcal{E}}_t^{*}}(s,a) \mathrm{d}\mu_{\mathcal{S}\times \mathcal{A}} \right]\nonumber\\
    &\leq\left(\sum_{h=1}^H\de_t\int_{\de_t[d_{h,\pi_{\mathcal{E}}^{*}}^{\bar{\mathcal{E}}_t}(s,a)]\neq 0}\frac{\alpha_\E^2\cdot d_{h,\pi_{\mathcal{E}}^{*}}^{\bar{\mathcal{E}}_t}(s,a)^2}{\de_t\left[d_{h,\pi_{\mathcal{E}}^{*}}^{\bar{\mathcal{E}}_t}(s,a)\right]}\right)^{1/2}\left(\sum_{h=1}^H\de_t\int_{\mathcal{S}\times \mathcal{A}}\de_t\left[d_{h,\pi_{\mathcal{E}}^{*}}^{\bar{\mathcal{E}}_t}(s,a)\right]\cdot\frac{\Delta_h^{\tilde{\E}_t^{*}}(s,a)^2}{\alpha_\E^2}\right)^{1/2}\nonumber\\
    &=\sqrt{\mathcal{I}^t\cdot\mathcal{T}^t}, \label{B10}
\end{align}
where the first equality is due to the fact that $\Delta_h^{\tilde{\mathcal{E}}_t^{*}}(s,a)$ is bounded ($\leq 2H$), and the second inequality is simply the Cauchy-Schwarz inequality with $\sum_{h}\de_t\int_{\mathcal{S}\times \mathcal{A}}$ as an ``integrated'' integral over the space $[H]\times\Theta\times\mathcal{S}\times\mathcal{A}$. Let us briefly discuss why the third equality holds. For the term $\mathcal{T}^t$, the derivation is straightforward, since $\de_t[X/\de_t[Y]]=\de_t[X]/\de_t[Y]$. For the term $\mathcal{I}^t$, first recall that $\de_t\left[d_{h,\pi_{\E}^{*}}^{\bar{\mathcal{E}}_t}(s,a)\right]=\de_t\left[d_{h,\ts^t}^{\bar{\mathcal{E}}_t}(s,a)\right]$ due to the property of TS. Then, we can use the independence between $d_{h,\ts^t}^{\bar{\mathcal{E}}_t}(s,a)$ and $\Delta_h^{\tilde{\E}_t^{*}}(s,a)$ given $\md_t$ to ``extract'' the expectation ($\de[XY]=\de[X]\cdot\de[Y]$ for $X,Y$ mutually independent):
\begin{align}
    \sum_{h=1}^H\de_t\int_{\mathcal{S}\times \mathcal{A}}\de_t\left[d_{h,\pi_{\mathcal{E}}^{*}}^{\bar{\mathcal{E}}_t}(s,a)\right]\cdot\frac{\Delta_h^{\tilde{\E}_t^{*}}(s,a)^2}{\alpha_\E^2}=\sum_{h=1}^H\de_t\int_{\mathcal{S}\times \mathcal{A}}d_{h,\ts^t}^{\bar{\mathcal{E}}_t}(s,a)\cdot\frac{\Delta_h^{\tilde{\E}_t^{*}}(s,a)^2}{\alpha_\E^2}=\mathcal{I}^t.
\end{align}
To summarize, by Eqn.~\eqref{B10} we have the following bound for $\Gamma_t^{\ts^t}$:
\begin{equation}\label{b13}
    \Gamma_t^{\ts^t}\leq\frac{\mathcal{I}^t\cdot\mathcal{T}^t}{\mathbb{I}_{t}^{\ts^t}\left(\tilde{\mathcal{E}}_{t}^{*};(\mathcal{H}_t,\mathcal{R}_{t,H})\right)}.
\end{equation}

\textbf{Step 3} (Bound $\mathcal{I}^t$). The key observation in this step is that $\mathcal{I}^t$ is in the form of total variation, and thus can be upper bounded by mutual information (in the form of KL divergence) by Pinsker's inequality. Specifically,
\begin{equation}
\begin{aligned}
  \mathcal{I}^t 
  &= \sum_{h=1}^{H} \mathbb{E}_t \mathbb{E}_{\pi_{\text{TS}}^t}^{\bar{\mathcal{E}}_t} \bigg( 
      \mathbb{E}_{(s',r') \sim (P_h^{\tilde{\mathcal{E}}_t^{*}} \otimes R_h^{\tilde{\mathcal{E}}_t^{*}})(\cdot|s_h,a_h)}
      \bigg[ 
          \frac{r' + V_{h+1,\pi_{\mathcal{E}}^{*}}^{\tilde{\mathcal{E}}_t^{*}}(s')}{\alpha_{\mathcal{E}}}
      \bigg] 
      - \mathbb{E}_{(s',r') \sim (P_h^{\bar{\mathcal{E}}_t} \otimes R_h^{\bar{\mathcal{E}}_t})(\cdot|s_h,a_h)}
      \bigg[ 
          \frac{r' + V_{h+1,\pi_{\mathcal{E}}^{*}}^{\tilde{\mathcal{E}}_t^{*}}(s')}{\alpha_{\mathcal{E}}} 
      \bigg] 
  \bigg)^2 \\
  &\overset{(a)}{=}\sum_{h=1}^{H} \mathbb{E}_t \mathbb{E}_{\pi_{\text{TS}}^t}^{\bar{\mathcal{E}}_t} \bigg( 
      \mathbb{E}_{(s',r') \sim (P_h^{\tilde{\mathcal{E}}_t^{*}} \otimes R_h^{\tilde{\mathcal{E}}_t^{*}})(\cdot|s_h,a_h)}
      \bigg[ 
          \frac{r' - r_h^{\inf}(s_h,a_h) + V_{h+1,\pi_{\mathcal{E}}^{*}}^{\tilde{\mathcal{E}}_t^{*}}(s') - \inf_s V_{h+1,\pi_{\mathcal{E}}^{*}}^{\tilde{\mathcal{E}}_t^{*}}(s')}{\alpha_{\mathcal{E}}}
      \bigg] \\
      &\qquad - \mathbb{E}_{(s',r') \sim (P_h^{\bar{\mathcal{E}}_t} \otimes R_h^{\bar{\mathcal{E}}_t})(\cdot|s_h,a_h)}
      \bigg[ 
          \frac{r' - r_h^{\inf}(s_h,a_h) + V_{h+1,\pi_{\mathcal{E}}^{*}}^{\tilde{\mathcal{E}}_t^{*}}(s') - \inf_s V_{h+1,\pi_{\mathcal{E}}^{*}}^{\tilde{\mathcal{E}}_t^{*}}(s')}{\alpha_{\mathcal{E}}} 
      \bigg] 
  \bigg)^2 \\
  &\overset{(b)}{\leq} \frac{1}{2} \sum_{h=1}^{H} \mathbb{E}_t \mathbb{E}_{\pi_{\text{TS}}^t}^{\bar{\mathcal{E}}_t} \left[
      D_{\text{KL}} \left( 
          \left( P_h^{\tilde{\mathcal{E}}_t^{*}} \otimes R_h^{\tilde{\mathcal{E}}_t^{*}} \right)(\cdot|s_h,a_h) \middle\| 
          \left( P_h^{\bar{\mathcal{E}}_t} \otimes R_h^{\bar{\mathcal{E}}_t} \right)(\cdot|s_h,a_h)
      \right) 
  \right]  \\
  &\overset{(c)}{\leq} \frac{1}{2} \mathbb{I}_{t}^{\pi_{\text{TS}}^t} \left( \tilde{\mathcal{E}}_{t}^{*}; (\mathcal{H}_t, \mathcal{R}_{t,H}) \right),
\end{aligned}
\end{equation}
where (a) adds and subtracts the constant term $r_h^{\inf}(s_h,a_h) $ and $ \inf_s V_{h+1,\pi_{\mathcal{E}}^{*}}^{\tilde{\mathcal{E}}_t^{*}}(s') $, (b) uses the Pinsker's inequality, (c) uses Lemma \ref{mutual}.
Plugging into Eqn.~\eqref{b13}, we derive that $\Gamma_t^{\pi_{\text{TS}}^t} \leq \frac{1}{2}\mathcal{T}^t$, and thus 
\begin{equation}\label{b49}
    BR_T(\ids)\leq\frac{1}{4\lambda}\de\left[\sum_{t=1}^T \mathcal{T}^t\right]+\frac{\lambda}{2}\log(K(\epsilon))+T\epsilon.
\end{equation}

\textbf{Step 4} (Bound $\mathbb{E}[\mathcal{T}^t]$). The analysis tools used in this step is the Doob's consistency theorem, with more details discussed in Appendix \ref{doob}. Define the true environment as $\E_0$. For brevity of notations, we define $$\mathcal{B}_{h,t}\triangleq\left\{(s,a)\in\mathcal{S}\times\mathcal{A}\middle| \mathbb{E}\left[d_{h,\pi_{\mathcal{E}}^{*}}^{\bar{\mathcal{E}}_t}(s,a) \right] \neq 0 \right\},$$ so that we can write $\mathcal{T}^t$ as
\begin{equation}
    \mathcal{T}^t=\sum_{h=1}^H\int_{(s,a)\in \mathcal{B}_{h,t}} \frac{\de_t\left[\alpha_\E^2\cdot d_{h,\pi_{\mathcal{E}}^{*}}^{\bar{\mathcal{E}}_t}(s,a)^2\right]}{\de_t\left[d_{h,\pi_{\mathcal{E}}^{*}}^{\bar{\mathcal{E}}_t}(s,a)\right]}\mathrm{d}\mu_{\mathcal{S}\times \mathcal{A}}.
\end{equation}
We now introduce another variable $\E'$ to apply Lemma \ref{strong}. Let $\E'\sim\dP(\cdot|\md_t)$ be a random variable  independent of $\E$. By definition of $d_{h,\pi_{\mathcal{E}}^{*}}^{\E}$, we have $\de_t\left[d_{h,\pi_{\mathcal{E}}^{*}}^{\bar{\mathcal{E}}_t}(s,a)\right]=\de_{\E,\E'\sim\dP_t(\cdot)}\left[d_{h,\pi_{\mathcal{E}}^{*}}^{\E'}(s,a)\right]$, and
\begin{align}
    \de_t\left[\alpha_\E^2\cdot d_{h,\pi_{\mathcal{E}}^{*}}^{\bar{\mathcal{E}}_t}(s,a)^2\right]&=\de_{\E\sim\dP_t(\cdot)}\left[\alpha_\E^2\cdot\de_{\E'\sim\dP_t(\cdot)}\left[d_{h,\pi_{\E}^{*}}^{\E'}(s,a)^2\right]\right]\nonumber\\
    &\leq \de_{\E,\E'\sim\dP_t(\cdot)}\left[\alpha_\E^2\cdot d_{h,\pi_{\E}^{*}}^{\E'}(s,a)^2\right],
\end{align}
where the inequality is due to the fact that $\de[X^2]\geq (\de[X])^2$. Therefore, we have
\begin{equation}
    \mathcal{T}^t\leq\sum_{h=1}^H\int_{\mathcal{S}\times\mathcal{A}}\frac{\de_t\left[\alpha_\E^2\cdot d_{h,\pi_{\mathcal{E}}^{*}}^{\E'}(s,a)^2\right]}{\de_t\left[d_{h,\pi_{\E}^*}^{\E'}(s,a)\right]}\chi_{\mathcal{B}_{h,t}}(s,a)\mathrm{d}\mu_{\mathcal{S}\times \mathcal{A}} =\sum_{h=1}^H\int_{\mathcal{S}\times \mathcal{A}} g_t(s,a,h,\mathcal{D}_t)\mathrm{d}\mu_{\mathcal{S}\times \mathcal{A}} ,
\end{equation}
where $\chi_A(\cdot)$ is the indicator function, i.e., $\chi_A(x)=1$ if $x\in A$; $\chi_A(x)=0$ if $x\notin A$. 
Since $d_{h,\pi}^{\mathcal{E}}(s,a)$ is assumed to be bounded, let $$M_d \triangleq \sup_{s,a,h,\pi,\E} d_{h,\pi}^{\mathcal{E}}(s,a) < \infty.$$ This implies $g_t(s,a,h,\mathcal{D}_t) \leq M_d H^2 $ and $\mathcal{T}^t \leq M_d H^3$.
By Lemma \ref{strong}, we have the following convergence results: for any $(s,a)\in\mathcal{S}\times \mathcal{A}$,
\begin{equation}
    \lim_{t \to \infty} \mathbb{E}_t\left[\alpha_{\E}^2d_{h,\pi_{\mathcal{E}}^{*}}^{\mathcal{E}'}(s,a)^2 \right] =\alpha_{\mathcal{E}_0}^2d_{h,\pi_{\mathcal{E}_0}^{*}}^{\mathcal{E}_0}(s,a)^2,
\end{equation}
\begin{equation}
    \lim_{t \to \infty} \mathbb{E}_t\left[d_{h,\pi_{\mathcal{E}}^{*}}^{\mathcal{E}'}(s,a) \right] = d_{h,\pi_{\mathcal{E}_0}^{*}}^{\mathcal{E}_0}(s,a),
\end{equation}
 and for any $x$,
\begin{equation}
\label{chi}
   \lim_{t \to \infty} \chi_{\mathcal{B}_{h,t}}(x) = \chi_{\mathcal{B}_h}(x),  \qquad  \text{where } \mathcal{B}_h\triangleq \left\{(s,a)\in\mathcal{S}\times\mathcal{A}\middle| \mathbb{E}\left[d_{h,\pi_{\mathcal{E}_0}^{*}}^{\E_0}(s,a) \right] \neq 0 \right\}.
\end{equation}
 From Eqn.~\eqref{chi}, we can also derive that $\lim_{t \to \infty} \chi_{\mathcal{B}_{h,t}\bigcap \mathcal{B}_h}(x)=1$, and $\lim_{t \to \infty} \chi_{\mathcal{B}_{h,t}\setminus \mathcal{B}_h}(x)=0$. By Lebesgue dominated convergence theorem, we have
\begin{align}
    \lim_{t\to\infty}\de\left[\mathcal{T}^t\middle|\E_0\right]&=\lim_{t\to\infty}\de_{\md_t\sim\dP(\cdot|\E_0)}\sum_{h=1}^H\int_{\mathcal{S}\times\mathcal{A}}\frac{\de_t\left[\alpha_\E^2\cdot d_{h,\pi_{\mathcal{E}}^{*}}^{\E'}(s,a)^2\right]}{\de_t\left[d_{h,\pi_{\E}^*}^{\E'}(s,a)\right]}\cdot\chi_{\mathcal{B}_{h,t}}(s,a)\mathrm{d}\mu_{\mathcal{S}\times \mathcal{A}}\nonumber\\
    &\leq \lim_{t\to\infty}\de_{\md_t\sim\dP(\cdot|\E_0)}\sum_{h=1}^H\int_{\mathcal{S}\times\mathcal{A}}\frac{\de_t\left[\alpha_\E^2\cdot d_{h,\pi_{\mathcal{E}}^{*}}^{\E'}(s,a)^2\right]}{\de_t\left[d_{h,\pi_{\E}^*}^{\E'}(s,a)\right]}\cdot\chi_{\mathcal{B}_{h,t}\bigcap\mathcal{B}_h}(s,a)\nonumber\\&\quad\quad+M_dH^2\lim_{t\to\infty}\sum_{h=1}^H\int_{\mathcal{S}\times\mathcal{A}}\chi_{\mathcal{B}_{h,t}\setminus\mathcal{B}_h}(s,a)\nonumber\\
    &=\sum_{h=1}^H\int_{\mathcal{S}\times\mathcal{A}}\de_{\md_t\sim\dP(\cdot|\E_0)}\left[\lim_{t\to\infty}\frac{\de_t\left[\alpha_\E^2\cdot d_{h,\pi_{\mathcal{E}}^{*}}^{\E'}(s,a)^2\right]}{\de_t\left[d_{h,\pi_{\E}^*}^{\E'}(s,a)\right]}\cdot\chi_{\mathcal{B}_{h,t}\bigcap\mathcal{B}_h}(s,a)\right]\nonumber\\
    &=\alpha_{\E_0}^2\cdot\int_{\mathcal{S}\times\mathcal{A}}d_{h,\pi_{\mathcal{E}_0}^{*}}^{\E_0}(s,a)\nonumber\\
    &\leq\alpha_{\E_0}^2\cdot H.
\end{align}
Thus, we have
\begin{equation}\label{b23}
    \lim_{t \to \infty} \mathbb{E}[\mathcal{T}^t]= \lim_{t \to \infty}\mathbb{E}[ \mathbb{E}[\mathcal{T}^t|\mathcal{E}_0]] = \mathbb{E}[ \lim_{t \to \infty} \mathbb{E}[\mathcal{T}^t|\mathcal{E}_0]] \leq \mathbb{E}[\alpha_{\mathcal{E}_0}^2 H] =\alpha^2 H.
\end{equation}

\textbf{Step 5:} By Eqn.~\eqref{b23}, we derive that there exists $T_0>0$ such that $\mathbb{E}[\mathcal{T}^t] \leq 2\alpha^2H$ for $t > T_0$.
Plugging into Eqn.~\eqref{b49} and then taking $\lambda=\sqrt{\alpha^2TH/{\log(K(\epsilon))}}$, we obtain 
\begin{align}\label{e}
    BR_T(\pi_{\text{r-IDS}})&\leq \frac{T\alpha^2H}{2\lambda}+ \frac{\lambda}{2}\log(K(\epsilon))+T\epsilon+T_0\nonumber\\
    &\leq\alpha \sqrt{TH\log(K(\epsilon))}+T\epsilon+T_0,
\end{align}
which finishes the proof of Theorem \ref{theorem}.
\end{proof}

\subsection{Proof of Theorem \ref{corollary1}}
\label{b3}
\begin{equation}
\begin{aligned}
&V_{1, \pi}^{\mathcal{E}}\big(s_1^{\ell}\big)-V_{1, \pi}^{\mathcal{E}'}\big(s_1^{\ell}\big)\\
&\overset{(a)}{=}\sum_{h=1}^H \mathbb{E}_{\pi}^{\mathcal{E}'}\bigg[\mathbb{E}_{s^{\prime} \sim P_h^{\mathcal{E}}\big(\cdot \mid s_h, a_h\big)}\left[V_{h+1, \pi}^{\mathcal{E}}\left(s^{\prime}\right)\right]-\mathbb{E}_{s^{\prime} \sim P_h^{\mathcal{E}'}\left(\cdot \mid s_h, a_h\right)}\big[V_{h+1, \pi}^{\mathcal{E}}\big(s^{\prime}\big)\big]\bigg] \\
&\quad +\sum_{h=1}^H \mathbb{E}_{\pi}^{\mathcal{E}'}\left[r_h^{\mathcal{E}}\left(s_h, a_h\right)-r_h^{\mathcal{E}'}\left(s_h, a_h\right)\right]\\
&\overset{(b)}{=} \sum_{h=1}^H \mathbb{E}_{\pi}^{\mathcal{E}'} \bigg[P_h^{\mathcal{E}}\big(\cdot \mid s_h, a_h\big)^T V_{h+1, \pi}^{\mathcal{E}}\big(\cdot \big) - P_h^{\mathcal{E}'}\big(\cdot \mid s_h, a_h\big)^T V_{h+1, \pi}^{\mathcal{E}'}\big(\cdot \big)   \bigg] \\
&\quad + \sum_{h=1}^H \mathbb{E}_{\pi}^{\mathcal{E}'} \bigg[P_h^{\mathcal{E}'}\big(\cdot \mid s_h, a_h\big)^T \big( V_{h+1, \pi}^{\mathcal{E}'}\big(\cdot \big) -  V_{h+1, \pi}^{\mathcal{E}}\big(\cdot \big) \big)   \bigg] +\sum_{h=1}^H \mathbb{E}_{\pi}^{\mathcal{E}'}\left[r_h^{\mathcal{E}}\left(s_h, a_h\right)-r_h^{\mathcal{E}'}\left(s_h, a_h\right)\right]\\
&\overset{(c)}{\leq} \sum_{h=1}^H \mathbb{E}_{\pi}^{\mathcal{E}'} \bigg[P_h^{\mathcal{E}}\big(\cdot \mid s_h, a_h\big)^T V_{h+1, \pi}^{\mathcal{E}}\big(\cdot \big) - P_h^{\mathcal{E}'}\big(\cdot \mid s_h, a_h\big)^T V_{h+1, \pi}^{\mathcal{E}'}\big(\cdot \big)   \bigg] \\
&\quad + \sum_{h=1}^H \mathbb{E}_{\pi}^{\mathcal{E}'} \bigg[ \Vert V_{h+1, \pi}^{\mathcal{E}'}\big(\cdot \big) -  V_{h+1, \pi}^{\mathcal{E}}\big(\cdot \big) \Vert_2   \bigg]  +\sum_{h=1}^H \mathbb{E}_{\pi}^{\mathcal{E}'}\left[r_h^{\mathcal{E}}\left(s_h, a_h\right)-r_h^{\mathcal{E}'}\left(s_h, a_h\right)\right],
\end{aligned}
\end{equation}
where (a) uses Lemma~\ref{lemma}, (b) adds and subtracts the term $P_h^{\mathcal{E}'}\big(\cdot \mid s_h, a_h\big)^T V_{h+1, \pi}^{\mathcal{E}'}\big(\cdot \big) $, (c) uses the Cauchy-schwartz inequality.

Since $P_h^{\mathcal{E}}\big(\cdot \mid s_h, a_h\big)^T V_{h+1, \pi}^{\mathcal{E}}\big(\cdot \big) \in [0,H]$,  we can  divide the value range  $[0,H]$ evenly into $\frac{3H^2}{\epsilon}$   parts. For each $(s,a,h)$, we  construct a covering set $\{ \mathcal{I}_{sah}^1,...,\mathcal{I}_{sah}^m \}$ for $[0,H]$ where $m=\frac{3H^2}{\epsilon}$. Each set is of length $\frac{\epsilon}{3H}$. Since $\Vert V_{h+1, \pi}^{\mathcal{E}}\big(\cdot \big) \Vert_2 \in [0,H\sqrt{S}]$, we construct a covering set $\{ \mathcal{J}_h^1,..., \mathcal{J}_h^{m'}\}$ for $[0,H\sqrt{S}]$ where $m'= \frac{6H^2\sqrt{S}}{\epsilon}$.
For reward function, we divide the value range $[0,1]$ evenly into $\frac{2H}{\epsilon}$ parts for all $ (s,a) \in \mathcal{S} \times \mathcal{A},h \in [H] $. The covering set is $\{ \mathcal{C}_1,...,\mathcal{C}_n \}$ where $n=\frac{3H}{\epsilon}$.
Then, we construct the partition $\{\Theta_k\}_{k=1}^K$ that $\mathcal{E} \in \Theta_k$ if for any $s,a,h$,
\[P_h^{\mathcal{E}}\big(\cdot \mid s, a\big)^T V_{h+1, \pi}^{\mathcal{E}}\big(\cdot \big) \in \mathcal{I}_{sah}^{k_1}, \quad  \Vert V_{h+1, \pi}^{\mathcal{E}}\big(\cdot \big) \Vert_2 \in \mathcal{J}_h^{k_2}, r_h^{\mathcal{E}}\left(s, a\right) \in \mathcal{C}_{k_3},  \]
where $k_1 \in [m], k_2 \in [m'], k_3 \in [n]$.

Therefore, $\{\Theta_k\}_{k=1}^K$ is a partition of $\Theta$. For any $k \in [K]$  and $\mathcal{E},\mathcal{E}' \in \Theta_k$, the following holds for any $s,a,h$,
\[ P_h^{\mathcal{E}}\big(\cdot \mid s, a\big)^T V_{h+1, \pi}^{\mathcal{E}}\big(\cdot \big) - P_h^{\mathcal{E}'}\big(\cdot \mid s, a\big)^T V_{h+1, \pi}^{\mathcal{E}'}\big(\cdot \big) \leq \frac{\epsilon}{3H}, \]
\[ \Vert V_{h+1, \pi}^{\mathcal{E}}\big(\cdot \big) - V_{h+1, \pi}^{\mathcal{E}'}\big(\cdot \big) \Vert_2 \leq \frac{\epsilon}{3H}, \]
\[ r_h^{\mathcal{E}}\left(s, a\right)-r_h^{\mathcal{E}'}\left(s, a\right) \leq \frac{\epsilon}{3H}.\]
Then, we have 
 \[ V_{1,\pi_{\mathcal{E}}^{*}}^{\mathcal{E}}(s_1^t)-V_{1,\pi_{\mathcal{E}}^{*}}^{\mathcal{E}'}(s_1^t) \leq \epsilon.\]
 
Since 
\[K(\epsilon)\leq  (\frac{3H^2}{\epsilon})^{SAH} \cdot (\frac{6H^2\sqrt{S}}{\epsilon})^{H} \cdot (\frac{3H}{\epsilon})^{SAH}, \]
we have 
\[ \log(K(\epsilon)) \leq SAH \log\bigg(\frac{3H^2}{\epsilon}\bigg) + H\log\bigg(\frac{6H^2\sqrt{S}}{\epsilon}\bigg) + SAH \log\bigg(\frac{3H}{\epsilon}\bigg) \leq 3SAH \log\bigg(\frac{6H^2\sqrt{S}}{\epsilon}\bigg).  \]
From Theorem \ref{theorem}, we have
\[ BR_T(\pi_{\text{r-IDS }}) \leq \alpha \sqrt{ 3SATH^2\log(\frac{6H^2\sqrt{S}}{\epsilon}) } +T\epsilon+T_0. \]

\subsection{Proof of Theorem \ref{corollary2}}
\label{b4}
Recall that $\mathcal{F}$ is the  compact feature space of $(\psi_h^{P})_i$ and $(\psi_h^{R})_i$. From the compactness of $\mathcal{F}$, there exists a finite $\epsilon$-covering number of $\mathcal{F}$. Let $M \triangleq \sup_{i,s} \max \{(\psi_h^{P}(s))_i (\psi_h^{R}(s))_i\}$.
Denote the  $\frac{\epsilon}{dMH^2}$-covering number of $\mathcal{F}$ as $K_{\mathcal{F}}(\epsilon)$.   We have $\mathcal{F}\subset \mathcal{K}_1 \bigcup ... \bigcup \mathcal{K}_{K_{\mathcal{F}}(\epsilon)} $ and for any $f,f' \in \mathcal{K}_i$, 
\[ \ell_g(f,f')=\int_{s} | \log \frac{f(s)}{f'(s)} | \leq \frac{\epsilon}{dMH^2}.  \]
Then we construct the partition of $\Theta$ as following:  $\mathcal{E}$ and $\mathcal{E}'$ belong to the same partition if and only if $(\psi_h^{P,\mathcal{E}})_i$ and $(\psi_h^{P,\mathcal{E}'})_i$ belong to the same partition of $\mathcal{F}$, $\forall i \in [d]$. Then, we have

\begin{align}
    &V_{1,\pi_{\mathcal{E}}^{*}}^{\E}(s_1)-V_{1,\pi_{\mathcal{E}}^{*}}^{\E'}(s_1) \nonumber\\
    &\overset{(a)}{=}\sum_{h=1}^{H}\mathbb{E}_{\pi_{\E}^{*}}^{\E'}\bigg[\mathbb{E}_{s'\sim P_h^{\E}(\cdot|s_h,a_h)}\left[ V_{h+1,\pi_{\mathcal{E}}^{*}}^{\E}(s')\right] -\mathbb{E}_{s'\sim P_h^{\E'}(\cdot|s_h,a_h)}\left[ V_{h+1,\pi_{\mathcal{E}}^{*}}^{\E}(s')\right]\bigg]+\sum_{h=1}^H\mathbb{E}_{\pi_{\E}^{*}}^{\E'}\left[ R_h^\E(s_h,a_h)-R_h^{\E'}(s_h,a_h)\right]\nonumber\\
    &\leq\sum_{h=1}^H\mathbb{E}_{\pi_{\E}^{*}}^{\E'}\bigg[\int_{\mathcal{S}}\left|P_h^\E(s'|s_h,a_h)-P_h^{\E'}(s'|s_h,a_h)\right|\cdot V_{h+1,\pi_{\E}^{*}}^{\E}(s')\mathrm{d}\mu_{\mathcal{S}}+\int_{[0,1]}\left|x\left(R_h^\E(x|s_h,a_h)-R_h^{\E'}(x|s_h,a_h)\right)\right|\mathrm{d}x\bigg]\nonumber\\
    &\overset{(c)}{\leq} H \sum_{h=1}^{H}\ell_1(P_h^{\E},P_h^{\E'})+ \sum_{h=1}^{H} \ell_1(R_h^{\E},R_h^{\E'}), 
\end{align}
where (a) uses Lemma~\ref{lemma}, (b) follows from $V_{h+1,\pi_{\E}^{*}}^{\E}(s') \leq H  $.
Next, we use the coverage of $\mathcal{F}$ under $\ell_g$ to bound $\ell_1(P_h^{\E},P_h^{\E'})$ and $\ell_1(R_h^{\E},R_h^{\E'}) $.

First  notice that
\[ \parallel \phi_h^{P}(s,a) \parallel_2 \leq 1 \Rightarrow |\phi_h^{P}(s,a)_i| \leq 1, \forall i \in [d] .\] 
For any $\mathcal{E},\mathcal{E}'$ that belong to the same partition, we have 
\begin{equation}
\label{linear1}
    \begin{aligned}
        \ell_1(P_h^{\mathcal{E}},P_h^{\mathcal{E}'})
        &=\sup_{s,a} \int_{s'} |P_h^{\mathcal{E}}(s'|s,a)-P_h^{\mathcal{E}'}(s'|s,a)| \\
        &= \sup_{s,a} \int_{s'} | \langle \phi_h^{P}(s,a) ,\psi_h^{P,\mathcal{E}}(s')-\psi_h^{P,\mathcal{E}'}(s')   \rangle |\\
        &\leq  \int_{s'} \sum_{i=1}^{d}|(\psi_h^{P,\mathcal{E}}(s')-\psi_h^{P,\mathcal{E}'}(s')   )_i| \\
        &\leq M\cdot \sum_{i=1}^d \ell_g((\psi_h^{P,\E})_i, (\psi_h^{P,\E'})_i )\\
        &\leq \frac{\epsilon}{H^2}.
    \end{aligned}
\end{equation}
The penultimate inequality uses the fact that $|a-b| \leq M|\log\frac{a}{b}|$ for any $a,b \in (0,M)$. By analogy, it can be concluded that $\ell_1(R_h^{\mathcal{E}},R_h^{\mathcal{E}'}) \leq \frac{\epsilon}{H^2}$.
Thus, 
\[ V_{1,\pi_{\mathcal{E}}^{*}}^{\E}(s_1)-V_{1,\pi_{\mathcal{E}}^{*}}^{\E'}(s_1) \leq 2\epsilon.  \]
Therefore, we get an $\epsilon$-value partition of $\Theta$.  Since $\psi_h^{P,\mathcal{E}}=( (\psi_h^{P,\mathcal{E}})_1,(\psi_h^{P,\mathcal{E}})_2,...,(\psi_h^{P,\mathcal{E}})_d )$ and each $(\psi_h^{P,\mathcal{E}})_i$ belongs to one of these $K_{\mathcal{F}}(\epsilon)$ sets ($\mathcal{K}_1,...,\mathcal{K}_{K_{\mathcal{F}}(\epsilon)})$,
the   number of this $\epsilon$-value partition can be bounded by $(K_{\mathcal{F}}(\epsilon))^{dH}$.
Thus, we have
\[  BR_T(\ids)\leq \alpha H \sqrt{d T  \log(K_{\mathcal{F}}(\epsilon))}+T\epsilon+T_0.  \]

\subsection{Proof of Proposition \ref{prop}}
\label{propproof}
\begin{unnumberedprop}
\label{prop2}
    Define 
    \begin{equation}
    r'_h(s,a)=r_h(s,a)+\frac{\lambda}{2}\mathbb{E}_t  \left[ \kl \left( (P_h^{\tilde{\mathcal{E}}_t^{*}} \otimes R_h^{\tilde{\mathcal{E}}_t^{*}})(\cdot|s_h, a_h) \middle\| (P_h^{\bar{\mathcal{E}}_t} \otimes R_h^{\bar{\mathcal{E}}_t})(\cdot|s_h, a_h) \right) \right].
\end{equation}
Then, for any policy $\pi$, we have
\begin{equation}
\label{lamep}
    \bigg| \mathbb{E}_{\pi}^{\bar{\mathcal{E}}_t}\bigg[ \sum_{h=1}^{H}r'_h(s_h,a_h) \bigg]-\mathbb{E}_{\pi}^{\bar{\mathcal{E}}_t}\bigg[ \sum_{h=1}^{H}\bar{r}_h(s_h,a_h) \bigg]\bigg|\leq \frac{\lambda}{2} \cdot \epsilon \cdot \left(1+\textcolor{blue}{\log\frac{B}{\beta}}\right).
\end{equation}
\end{unnumberedprop}

\begin{proof}[Proof]
    We begin our proof by calculating the difference of the two KL divergence terms. Recall that
    \[ \bar{r}_h(s_h,a_h)= r_h(s,a)+\frac{\lambda}{2}\mathbb{E}_t  \left[ \kl \left( (P_h^{\mathcal{E}} \otimes R_h^{\mathcal{E}})(\cdot|s_h, a_h) \middle\| (P_h^{\bar{\mathcal{E}}_t} \otimes R_h^{\bar{\mathcal{E}}_t})(\cdot|s_h, a_h) \right) \right].\]
    First, by triangle inequality, we have
\begin{align}\label{y}
        &\bigg|D_{\text{KL}} \left( (P_h^{\tilde{\mathcal{E}}_t^{*}} \otimes R_h^{\tilde{\mathcal{E}}_t^{*}})(\cdot|s_h, a_h) \middle\| (P_h^{\bar{\mathcal{E}}_t} \otimes R_h^{\bar{\mathcal{E}}_t})(\cdot|s_h, a_h) \right)-D_{\text{KL}} \left( (P_h^{\mathcal{E}} \otimes R_h^{\mathcal{E}})(\cdot|s_h, a_h) \middle\| (P_h^{\bar{\mathcal{E}}_t} \otimes R_h^{\bar{\mathcal{E}}_t})(\cdot|s_h, a_h) \right) \bigg|\nonumber\\
        &\leq \bigg|\int_{\mathcal{S}} P_h^{\tilde{\mathcal{E}}_t^{*}}(x|s_{h},a_{h})\log \frac{P_h^{\tilde{\mathcal{E}}_t^{*}}(x|s_{h},a_{h})}{P_h^{\bar{\mathcal{E}}_t} ( x|s_{h},a_{h}) }\mathrm{d}x -  \int_{\mathcal{S}} P_h^{\mathcal{E}}(x|s_{h},a_{h})\log \frac{P_h^{\mathcal{E}}(x|s_{h},a_{h})}{P_h^{\bar{\mathcal{E}}_t} ( x|s_{h},a_{h}) }\mathrm{d}x \bigg| \nonumber\\
        &+ \bigg|\int_{[0,1]} R_h^{\tilde{\mathcal{E}}_t^{*}}(x|s_{h},a_{h})\log \frac{R_h^{\tilde{\mathcal{E}}_t^{*}}(x|s_{h},a_{h})}{R_h^{\bar{\mathcal{E}}_t} ( x|s_{h},a_{h}) }\mathrm{d}x -  \int_{[0,1]}R_h^{\mathcal{E}}(x|s_{h},a_{h})\log \frac{R_h^{\mathcal{E}}(x|s_{h},a_{h})}{R_h^{\bar{\mathcal{E}}_t} ( x|s_{h},a_{h}) } \mathrm{d}x\bigg|.
\end{align}
Let $o=(s_{h},a_{h})$. For the first term in Eqn.~\eqref{y}, we have the following bound
\begin{align} 
\label{twoterm2}
        &\bigg|\int_{\mathcal{S}} P_h^{\tilde{\mathcal{E}}_t^{*}}(x|o)\log \frac{P_h^{\tilde{\mathcal{E}}_t^{*}}(x|o)}{P_h^{\bar{\mathcal{E}}_t} ( x|o) } -  P_h^{\mathcal{E}}(x|o)\log \frac{P_h^{\mathcal{E}}(x|o)}{P_h^{\bar{\mathcal{E}}_t} ( x|o) }\mathrm{d}x   \bigg| \nonumber\\
        &\leq \bigg|\int_{\mathcal{S}} P_h^{\tilde{\mathcal{E}}_t^{*}}(x|o)\log \frac{P_h^{\tilde{\mathcal{E}}_t^{*}}(x|o)}{P_h^{\bar{\mathcal{E}}_t} (x|o) } -  P_h^{\mathcal{E}}(x|o)\log \frac{P_h^{\tilde{\mathcal{E}}_t^{*}}(x|o)}{P_h^{\bar{\mathcal{E}}_t} (x|o) }\mathrm{d}x\bigg| \nonumber \\
        &\qquad +\bigg|\int_{\mathcal{S}} P_h^{\mathcal{E}}(x|o)\log \frac{P_h^{\tilde{\mathcal{E}}_t^{*}}(x|o)}{P_h^{\bar{\mathcal{E}}_t} ( x|o) }-  P_h^{\mathcal{E}}(x|o)\log \frac{P_h^{\mathcal{E}}(x|o)}{P_h^{\bar{\mathcal{E}}_t} (x|o) }\mathrm{d}x \bigg|\nonumber\\
        &\leq \int_{\mathcal{S}} \left|P_h^{\tilde{\mathcal{E}}_t^{*}}(x|o)- P_h^{\mathcal{E}}(x|o)\right| \cdot \bigg|\log \frac{P_h^{\tilde{\mathcal{E}}_t^{*}}(x|o)}{P_h^{\bar{\mathcal{E}}_t} ( x|o) }  \bigg|\mathrm{d}x + \int_{\mathcal{S}} B\cdot \bigg|\log \frac{P_h^{\tilde{\mathcal{E}}_t^{*}}(x|o)}{P_h^{\mathcal{E}} ( x|o) } \bigg|\mathrm{d}x\nonumber\\
        &\leq B\cdot\left(1+\textcolor{blue}{\log\frac{B}{\beta}}\right) \int_{\mathcal{S}} \bigg|\log \frac{P_h^{\tilde{\mathcal{E}}_t^{*}}(x|o)}{P_h^{\mathcal{E}} ( x|o) } \bigg|\mathrm{d}x\nonumber\\
        &\leq \frac{(1+\textcolor{blue}{\log(B/\beta)})\epsilon }{2H^2},
\end{align}
where the first inequality is due to triangle inequality; the second inequality again uses triangle inequality, and the fact that $P_h^{\mathcal{E}}(x|o) \leq \textcolor{blue}{ B}$; the third inequality is due to the fact that $|a-b| \leq B\cdot|\log\frac{a}{b}|$ for any $a,b\in(0,B)$ and $\bigg|\log \frac{P_h^{\tilde{\mathcal{E}}_t^{*}}(\cdot|o)}{P_h^{\bar{\mathcal{E}}_t} ( \cdot|o) }  \bigg| \leq \textcolor{blue}{\log\frac{B}{\beta}}$. Note that when $P_h^{\bar{\mathcal{E}}_t}(\cdot|o)$ equals  zero, since $\bar{\mathcal{E}}_t$ is the mean MDP, it turns out that $P_h^{\tilde{\mathcal{E}}_t^{*}}(\cdot|o)$ also equals  zero, yielding $\bigg|\log \frac{P_h^{\tilde{\mathcal{E}}_t^{*}}(\cdot|o)}{P_h^{\bar{\mathcal{E}}_t} ( \cdot|o) }  \bigg|=0.$ Finally, the last inequality is due to Eqn.~\eqref{4_3}.

\textcolor{blue}{If we employed the KL divergence or $\ell_1$ metric, the derivation of third inequality would become invalid.  }

For the second term in Eqn.~\eqref{y}, adopting a similar approach, we have
\begin{align} 
\label{twoterm3}
        &\bigg|\int_{[0,1]} R_h^{\tilde{\mathcal{E}}_t^{*}}(x|o)\log \frac{R_h^{\tilde{\mathcal{E}}_t^{*}}(x|o)}{R_h^{\bar{\mathcal{E}}_t} ( x|o) } -  R_h^{\mathcal{E}}(x|o)\log \frac{R_h^{\mathcal{E}}(x|o)}{R_h^{\bar{\mathcal{E}}_t} ( x|o) }\mathrm{d}x   \bigg| \nonumber\\
        &\leq \bigg|\int_{[0,1]} R_h^{\tilde{\mathcal{E}}_t^{*}}(x|o)\log \frac{R_h^{\tilde{\mathcal{E}}_t^{*}}(x|o)}{R_h^{\bar{\mathcal{E}}_t} (x|o) } -  R_h^{\mathcal{E}}(x|o)\log \frac{R_h^{\tilde{\mathcal{E}}_t^{*}}(x|o)}{R_h^{\bar{\mathcal{E}}_t} (x|o) }\mathrm{d}x\bigg| \nonumber\\
        &\qquad + \bigg|\int_{[0,1]} R_h^{\mathcal{E}}(x|o)\log \frac{R_h^{\tilde{\mathcal{E}}_t^{*}}(x|o)}{R_h^{\bar{\mathcal{E}}_t} ( x|o) }-  R_h^{\mathcal{E}}(x|o)\log \frac{R_h^{\mathcal{E}}(x|o)}{R_h^{\bar{\mathcal{E}}_t} (x|o) }\mathrm{d}x \bigg|\nonumber\\
        &\leq \int_{[0,1]} \left|R_h^{\tilde{\mathcal{E}}_t^{*}}(x|o)- R_h^{\mathcal{E}}(x|o)\right| \cdot \bigg|\log \frac{R_h^{\tilde{\mathcal{E}}_t^{*}}(x|o)}{R_h^{\bar{\mathcal{E}}_t} ( x|o) }  \bigg|\mathrm{d}x + \int_{[0,1]} B\cdot \bigg|\log \frac{R_h^{\tilde{\mathcal{E}}_t^{*}}(x|o)}{R_h^{\mathcal{E}} ( x|o) } \bigg|\mathrm{d}x\nonumber\\
        &\leq B\cdot\left(1+ \textcolor{blue}{\log\frac{B}{\beta}}\right) \int_{[0,1]} \bigg|\log \frac{R_h^{\tilde{\mathcal{E}}_t^{*}}(x|o)}{R_h^{\mathcal{E}} ( x|o) } \bigg|\mathrm{d}x\nonumber\\
        &\leq \frac{(1+ \textcolor{blue}{\log(B/\beta)})\epsilon }{2 H}.
\end{align}
Hence, adding up Eqn.~\eqref{twoterm2} and Eqn.~\eqref{twoterm3}, we obtain
$$|r'_h(s_h,a_h)-\bar{r}_h(s_h,a_h)| \leq  \frac{\lambda}{2}\cdot\frac{(1+\log(B/\beta))\epsilon }{H}. $$ 
Finally, summing over $h\in[H]$, we have 
\begin{equation}
    \bigg| \mathbb{E}_{\pi}^{\bar{\mathcal{E}}_t}\bigg[ \sum_{h=1}^{H}r'_h(s_h,a_h) \bigg]-\mathbb{E}_{\pi}^{\bar{\mathcal{E}}_t}\bigg[ \sum_{h=1}^{H}\bar{r}_h(s_h,a_h) \bigg]\bigg|\leq \frac{\lambda}{2}\cdot \epsilon \cdot (1+\textcolor{blue}{\log\frac{B}{\beta}})
\end{equation}
The proof is finished.
\end{proof}

\subsection{Proof of Proposition \ref{prop_1}}
\label{propproof_1}
The proof follows essentially the same structure as that of Proposition 1, with the only difference lying in the third inequality of Eq.~\eqref{twoterm2}.
By applying the Mean Value Theorem, we have 
\[ \int_{\mathcal{S}} \bigg|\log \frac{P_h^{\tilde{\mathcal{E}}_t^{*}}(x|o)}{P_h^{\mathcal{E}} ( x|o) } \bigg|\mathrm{d}x \leq 
 \frac{1}{\beta}\int_{\mathcal{S}} \bigg| P_h^{\tilde{\mathcal{E}}_t^{*}}(x|o)- P_h^{\mathcal{E}} ( x|o) \bigg|\mathrm{d}x = \frac{1}{\beta} \ell_1(P_h^{\tilde{\mathcal{E}}_t^{*}},P_h^{\mathcal{E}}) \leq \frac{1}{\beta} \cdot \text{constant} \cdot \epsilon. \]
 We can derive the third inequality in Eq.~\eqref{twoterm2} as 
 \[ B\cdot\left(1+\frac{B}{\beta}\right) \int_{\mathcal{S}} \bigg|\log \frac{P_h^{\tilde{\mathcal{E}}_t^{*}}(x|o)}{P_h^{\mathcal{E}} ( x|o) } \bigg|\mathrm{d}x \leq \frac{1}{\beta} \cdot \text{constant} \cdot \epsilon. \] 
In the conclusion of Proposition 1, the term involving $\beta$ has been modified from $\log \frac{1}{\beta}$ to $\frac{1}{\beta}$. Therefore, with a suitably chosen partition radius, we have
\begin{equation}
    \bigg| \mathbb{E}_{\pi}^{\bar{\mathcal{E}}_t}\bigg[ \sum_{h=1}^{H}r'_h(s_h,a_h) \bigg]-\mathbb{E}_{\pi}^{\bar{\mathcal{E}}_t}\bigg[ \sum_{h=1}^{H}\bar{r}_h(s_h,a_h) \bigg]\bigg|\leq \frac{\lambda}{2}\cdot \epsilon \cdot (1+\frac{B}{\beta})
\end{equation}

\section{Basic Properties of the Measure $\ell_g$}\label{lg}

The following result deals with the problem of convexity.  Let $B(\mathcal{C},\epsilon)$ be an $\epsilon$-ball with its center at $\mathcal{C}$.
Note that  $B(\mathcal{C},\epsilon)$ is not essentially convex under $\ell_g$: for $P,Q\in B(\mathcal{C},\epsilon)$ and $\lambda\in(0,1)$, it does not hold that $\lambda P+(1-\lambda)Q\in B(\mathcal{C},\epsilon)$. However, using Lemma \ref{metric}, we have the following result:
\begin{lemma}\label{center}
    For any $P,Q\in B(\mathcal{C},\epsilon)$ and $\lambda\in[0,1]$, \textcolor{blue}{$\lambda P+(1-\lambda)Q$} lies in the $(2\epsilon)$-ball at $\mathcal{C}$, i.e., $$\ell_g(\lambda P+(1-\lambda)Q, \mathcal{C})\leq 2\epsilon.$$
\end{lemma}
\begin{proof}
    By definition of $\ell_g$, we have
    \begin{align*}
        \ell_g(\lambda P +(1-\lambda)Q,\mathcal{C}) &= \sup_{o} \int_x \bigg| \log\frac{  \lambda P(x|o)+(1-\lambda) Q(x|o)}{\mathcal{C}(x|o)}\bigg|\\
        &\leq\sup_{o} \int_x \bigg| \log \frac{P(x|o)}{\mathcal{C}(x|o)} + \log \frac{Q(x|o)}{\mathcal{C}(x|o)}  \bigg|\\
        &\leq 2\epsilon,
    \end{align*}
    where the first inequality uses the fact that $| \log(\lambda a+(1-\lambda)b)| \leq |\log a|+ |\log b|$ for any $a,b > 0$ and $\lambda \in [0,1]$; the second inequality is due to $P,Q\in B(\mathcal{C},\epsilon)$. This finishes the proof of Lemma \ref{center}.
\end{proof}

\subsection{\textcolor{blue}{A Counterexample to Demonstrate Non-convex }}
\textcolor{blue}{
 We present the following counterexample to demonstrate that the unit ball under  $\ell_g$ is not convex. Let $X $ be the set of positive and Lebesgue  integrable functions  defined on $[0,1]$. We show that the ball centered at $f_0 \in X$ 
    \[ B(f_{0},r)=\bigl\{f\in X:\,\ell_g(f,f_{0})\le r\bigr\} \]
    is not convex.
    Let 
    \[
    f(x)=\begin{cases}
    5      & 0\le x\le 0.1,\\
    \frac{5}{9}      & 0.1< x\le 1,
    \end{cases}
    \qquad
    g(x)=\begin{cases}
    \frac{5}{9}      & 0\le x\le 0.9,\\
    5      & 0.9 < x\le 1,
    \end{cases}
    \qquad
    f_0(x)= 1.
    \]
    It can be verified:
    \[
    \int_{0}^{1}f(x)\,dx=0.1\cdot 5+0.9\cdot\frac59=1,
    \]
    and the same holds for $g$. Hence $f,g$ can be seen as the two probability measures.
    We also have
    \[
    \ell_g(f,f_{0}) =0.1\log 5+0.9\log\!\Bigl(\tfrac95\Bigr)
               \approx 0.68995.
    \]
    The same holds for $\ell_g(g,f_0)$.
    And
    \[ \ell_g(\frac{f+g}{2},f_0)=\ell_g(\frac{25}{9},1)=|\log \frac{25}{9}| \approx 1.02165. \]
    Let $r=1$, then
    $ \ell_g(f,f_0)=\ell_g(g,f_0) \approx 0.68995<r, \ell_g(\frac{f+g}{2},f_0) \approx 1.02165 >r $.\\
    So, $f,g \in B(f_0,r)$ but $ \frac{f+g}{2} \notin B(f_0,r)$.}

\section{Technical Lemmas}
\subsection{Value Function and Mutual Information}
We cite the following lemma from ~\cite{moradipari2023improved}. Similar results can be found in ~\cite{osband2013more,foster2021statistical}.
\begin{lemma}
    \label{lemma}
    For any two environments $\mathcal{E},\mathcal{E}'$ with potentially different transition and reward functions, and any policy $\pi$, we have
    \begin{equation}
    \begin{aligned}
        V_{1,\pi}^{\mathcal{E}}(s_1)-V_{1,\pi}^{\mathcal{E}'}(s_1) =&\sum_{h=1}^{H}\mathbb{E}_{\pi}^{\mathcal{E}'}\bigg[  \mathbb{E}_{(s',r') \sim (P_h^{\mathcal{E}} \otimes R_h^{\mathcal{E}} )(\cdot|s_h,a_h)}\left[ r'+V_{h+1,\pi}^{\mathcal{E}}(s')\right] \\
        &-\mathbb{E}_{(s',r') \sim (P_h^{\mathcal{E}'}\otimes R_h^{\mathcal{E}'})(\cdot|s_h,a_h)}\left[ r'+V_{h+1,\pi}^{\mathcal{E}}(s') \right] \bigg].
        \end{aligned}
    \end{equation}
\end{lemma}

\begin{lemma}
\label{mutual}
The mutual information $\mathbb{I}_{t}^{\pi_{\text{TS}}^t}\left(\tilde{\mathcal{E}}_t^{*}; (\mathcal{H}_{t},\mathcal{R}_{t,H})\right)$ can be lower bounded as
    \begin{equation}
    \label{tmp3}
    \begin{aligned}
        \mathbb{I}_{t}^{\pi_{\text{TS}}^t}\left(\tilde{\mathcal{E}}_t^{*}; (\mathcal{H}_{t},\mathcal{R}_{t,H})\right) \geq \sum_{h=1}^H \mathbb{E}_t \left[ \mathbb{E}_{\pi_{\text{TS}}^t}^{\bar{\mathcal{E}}_t} \left[ D_{\text{KL}} \left( (P_h^{\tilde{\mathcal{E}}_t^{*}} \otimes R_h^{\tilde{\mathcal{E}}_t^{*}})(\cdot|s_h, a_h) \middle\| (P_h^{\bar{\mathcal{E}}_t} \otimes R_h^{\bar{\mathcal{E}}_t}) (\cdot|s_h, a_h) \right) \right] \right].
    \end{aligned}
\end{equation}
\end{lemma}
\begin{proof}[Proof of Lemma \ref{mutual}]
    Using the chain rule of mutual information,
     \begin{equation}
    \begin{aligned}
        &\mathbb{I}_{t}^{\pi_{\text{TS}}^t}\left(\tilde{\mathcal{E}}_t^{*}; (\mathcal{H}_{t},\mathcal{R}_{t,H})\right) \\ 
        &=
        \sum_{h=1}^H \mathbb{I}_{t}^{\pi_{\text{TS}}^t} \left( \tilde{\mathcal{E}}_t^{*}; (s_h^{t,1}, a_h^{t,1},r_h^{t,1},s_h^{t,0}, a_h^{t,0},r_h^{t,0}) \mid (\mathcal{H}_{t, h-1}, \mathcal{R}_{t,h-1} ) \right) 
         + \mathbb{I}_{t}^{\pi_{\text{TS}}^t} \left( \tilde{\mathcal{E}}_t^{*}; o_{t} \mid (\mathcal{H}_{t,H},\mathcal{R}_{t,H}) \right) \\
        &= \sum_{h=1}^H \mathbb{I}_{t}^{\pi_{\text{TS}}^t}\big(\tilde{\mathcal{E}}_t^{*}; s_h^{t,1}\mid \mathcal{H}_{t,h-1},\mathcal{R}_{t,h-1}\big) + \sum_{h=1}^H \mathbb{I}_{t}^{\pi_{\text{TS}}^t}\big(\tilde{\mathcal{E}}_t^{*}; a_h^{t,1} \mid  s_h^{t,1}, \mathcal{H}_{t,h-1},\mathcal{R}_{t,h-1}\big) \\
       &\quad  + \sum_{h=1}^H \mathbb{I}_{t}^{\pi_{\text{TS}}^t}\big(\tilde{\mathcal{E}}_t^{*}; r_h^{t,1} \mid  s_h^{t,1},  s_h^{t,1}, \mathcal{H}_{t,h-1},\mathcal{R}_{t,h-1}\big) + \sum_{h=1}^H \mathbb{I}_{t}^{\pi_{\text{TS}}^t}\big(\tilde{\mathcal{E}}_t^{*}; s_h^{t,0} \mid  s_h^{t,1},  a_h^{t,1}, r_h^{t,1},\mathcal{H}_{t,h-1},\mathcal{R}_{t,h-1}\big) \\
       &\quad + \sum_{h=1}^H \mathbb{I}_{t}^{\pi_{\text{TS}}^t}\big(\tilde{\mathcal{E}}_t^{*}; a_h^{t,0} \mid  s_h^{t,1},  a_h^{t,1}, r_h^{t,1},s_h^{t,0},\mathcal{H}_{t,h-1},\mathcal{R}_{t,h-1}\big)  \\
       &\quad +\sum_{h=1}^H\mathbb{I}_{t}^{\pi_{\text{TS}}^t}\big(\tilde{\mathcal{E}}_t^{*}; r_h^{t,0} \mid  s_h^{t,1},  a_h^{t,1}, r_h^{t,1}, s_h^{t,0}, a_h^{t,0},\mathcal{H}_{t,h-1},\mathcal{R}_{t,h-1}\big)
        + \mathbb{I}_{t}^{\pi_{\text{TS}}^t}\left(\tilde{\mathcal{E}}_t^{*}; o_{t} \mid (\mathcal{H}_{t,H},\mathcal{R}_{t,H})\right).
    \end{aligned}
\end{equation}
From ~\cite{moradipari2023improved}, 
the first three terms on the right side of the above equation are equal to 
\begin{equation}
\label{temp1}
     \sum_{h=1}^H \mathbb{E}_t \left[ \mathbb{E}_{\pi_{\text{TS}}^t}^{\bar{\mathcal{E}}_t} \left[ D_{\text{KL}} \left( (P_h^{\tilde{\mathcal{E}}_t^{*}} \otimes r_h^{\tilde{\mathcal{E}}_t^{*}})(\cdot|s_h, a_h) \middle\| (P_h^{\bar{\mathcal{E}}_t} \otimes r_h^{\bar{\mathcal{E}}_t}) (\cdot|s_h, a_h) \right) \right] \right].
\end{equation}
Based on the non-negativity of mutual information, we obtain the conclusion of the lemma.

\end{proof}

\subsection{Posterior Consistency}\label{doob}
\begin{lemma}
\label{strong}
    Assume that there exists a strongly consistent estimator of the true environment given the history. Let $\Pi$ be some measure. For any  $\Pi$-integrable function $f:\Theta \rightarrow \mathbb{R} $ and almost every $\mathcal{D}_{\infty}$ sampled from the true environment $\mathcal{E}_0$, we have
    \[ \lim_{t\to \infty} \mathbb{E}_t\big[ f(\mathcal{E}) \big] = f(\mathcal{E}_0). \]
    And if $f:\Theta \times \Theta \rightarrow \mathbb{R}$ is bounded and $(\Pi \times \Pi)$-integrable, for almost every $\mathcal{D}_{\infty}$ sampled from the true environment $\mathcal{E}_0$, we have 
    \[ \lim_{t\to \infty} \mathbb{E}_t\big[ f(\mathcal{E},\mathcal{E}') \big] = f(\mathcal{E}_0,\mathcal{E}_0), \]
    where the expectation is taken over all $\mathcal{E}$ and $\mathcal{E}'$.
\end{lemma}
We refer the readers to  Theorem 6.9 in~\cite{ghosal2017fundamentals} or Appendix K in~\cite{moradipari2023improved} for the definition of a strongly consistent estimator and for more details of the proof.

\end{document}